\documentclass[preprint]{siamart251216}

\usepackage[T1]{fontenc}
\usepackage[utf8]{inputenc}
\usepackage[english]{babel}
\usepackage[activate={true,nocompatibility}, final, kerning=true, spacing=true]{microtype}

\usepackage{enumerate}

\usepackage{hyperref}

\usepackage{amsmath, amssymb, amsfonts}
\usepackage{bookmark}

\def\RR{\mathbb{R}}

\def\NN{\mathbb{N}}

\def\XX{\mathbb{X}}
\def\UU{\mathbb{U}}

\def\EE{\mathbb{E}}
\def\PP{\mathbb{P}}

\def\cC{\mathcal{C}}
\def\cD{\mathcal{D}}
\def\cE{\mathcal{E}}

\def\cK{\mathcal{K}}
\def\cM{\mathcal{M}}
\def\cN{\mathcal{N}}
\def\cL{\mathcal{L}}
\def\cO{\mathcal{O}}
\def\cP{\mathcal{P}}

\def\cT{\mathcal{T}}
\def\cU{\mathcal{U}}

\def\cW{\mathcal{W}}

\def\cZ{\mathcal{Z}}

\def\fC{\mathfrak{C}}

\def\fc{\mathfrak{c}}

\def\Tr{\mathrm{Tr}}

\def\e{\mathrm{e}}
\def\de{\mathrm{d}}

\let\epsilon\varepsilon
\let\phi\varphi

\newcommand{\inner}[2]{\langle #1, #2 \rangle}

\DeclareRobustCommand{\rchi}{{\mathpalette\irchi\relax}}
\newcommand{\irchi}[2]{\scalebox{1.1}{\raisebox{\depth}{$#1\chi$}}}

\DeclareMathOperator*{\argmin}{arg\,min}

\DeclareMathOperator{\Prob}{Prob}

\usepackage{enumitem}

\usepackage{algpseudocodex}
\usepackage{algorithm}

\usepackage{tikz}
\usepackage{pgfplots}
\usetikzlibrary{arrows.meta, positioning, shapes.geometric, backgrounds}
\pgfplotsset{compat=1.18}

\usepackage{array}
\allowdisplaybreaks

\newsiamremark{remark}{Remark}
\newsiamthm{assumption}{Assumption}

\usepackage{enumitem}

\usepackage{algpseudocodex}
\usepackage{algorithm}

\usepackage{tikz}
\usepackage{pgfplots}
\usetikzlibrary{arrows.meta, positioning, shapes.geometric, backgrounds}
\pgfplotsset{compat=1.18}

\usepackage[dvipsnames]{xcolor}
\newcommand{\ka}[1]{{#1}}

\usepackage{array}
\title{Generalization Bounds on Optimal Control for Transformer Training and Wasserstein Distributional Robustness\thanks{This is the preprint version.\funding{This work has been partially supported by T\"UB\.ITAK through the 2211-E program, by Turkish Academy of Sciences through the 2025 GEBIP Award, and by the Natural Sciences and Engineering Research Council of Canada (NSERC).}}}

\author{Ka\u{g}an Akman\thanks{Department of Mathematics, Bilkent University, 06800, Ankara,
Turkey (\email{kagan.akman@bilkent.edu.tr}).}
\and Naci Saldi\thanks{Department of Mathematics, Bilkent University, 06800, Ankara,
Turkey (\email{naci.saldi@bilkent.edu.tr}).}
\and Serdar Y\"uksel\thanks{Department of Mathematics and Statistics, Queen's University, Kingston, K7L 3N8, Ontario, Canada (\email{yuksel@queensu.ca})}
}

\begin{document}

\maketitle

\begin{abstract}
    We derive finite-sample generalization bounds for Transformers trained with dynamic programming recursions. Building on the doubly lifted, measure-valued formulation of Transformer dynamics \cite{akman2026optimal}, we view data sets as probability laws on pairs of empirical input-output measures, allowing us to interpret the training problem as a finite-horizon Markovian control problem. We then analyze a quantized model, derived by quantizing the state, action, and measure-state spaces, and derive explicit finite-sample generalization bounds using concentration inequalities for empirical laws on finite metric spaces together with a Lipschitz stability estimate for the value function. These bounds are transferred to the base model at the cost of an explicit approximation error. Finally, we show that the same machinery yields a distributionally robust control formulation of the training problem, connecting Transformer generalization to Wasserstein distributionally robust optimization.
\end{abstract}

\begin{keywords}
    Transformers, generalization bounds, Markov decision processes,  concentration inequalities, distributionally robust control
\end{keywords}

\begin{MSCcodes}
Primary 93E20; Secondary 68T07, 90C17, 49Q22, 49L20, 60B10, 49J45.
\end{MSCcodes}

\section{Introduction}
Transformers, introduced in \cite{vaswani2017attention}, constitute an efficient and widely deployed class of neural networks, often used to create large language models. Transformer models are trained using finite data sets assumed to be independently sampled from an identical distribution while they are mostly deployed for data-generating applications that are not fully represented by those samples. This raises the fundamental question of generalization: {\em How far does the performance of a model trained on the empirical data transfer to the underlying true law?}

The mismatch error is defined to be the difference between the performance of the policy learned (from training based on empirical data) when applied to the true distribution, and the optimal cost under the true distribution. Equivalently, this question can be seen as a problem of robustness to perturbations of initial law. In particular, this problem is structural for Transformers as the evolution of tokens explicitly depends on the empirical distribution of the sequence they belong to.

In this paper, building on \cite{akman2026optimal}, we develop a framework for Transformer generalization in which Transformers are not viewed as members of a hypothesis class, but a discrete-time controlled dynamical system whose control variables are the layer-wise training weights. In the doubly lifted model, data sets are raised to empirical measures on pairs of empirical input-output measures. This setting poses the training problem as finding a control trajectory satisfying an optimization objective modeling mapping an input (feature vector) to the correct output (label vector). Generalization is then formulated as bounding the excess risk of the empirical training minimizer; the gap between its performance under the true law and the optimal value under that law. The finite-sample bounds we obtain come from a Lipschitz stability estimate for the value function, and a concentration inequality for the empirical lifted law.

For Transformers, this problem presents itself with several difficulties. First, the dynamics of a single token are not Markovian; it depends not only on the current realization of the token and chosen weights, but also, on the empirical measure of the whole sequence. To restore Markovian property, the token-wise dynamics is lifted to the realm of probability measures, however, this leads to an optimization problem over infinite-dimensional vector spaces. To obtain explicit and computable bounds, we pass to a quantized model, and then transfer the bounds to the base model at the cost of approximation error. Finally, in Section~\ref{sec:dr}, the same machinery used to derive finite-sample bounds admits a complementary angle: the training problem can be reconsidered as a Wasserstein distributionally robust control problem, or equivalently, a zero-sum game in which the controller selects a weight trajectory while an adversary counters by selecting a nearby data generating law.

\subsection{Related Work}
Here, we discuss the relevant literature under three branches.
\smallskip

\noindent{\bf Generalization in statistical learning:} The generalization theory has been studied from many complementary angles with the common aim being to control the gap between true risk and empirical risk by a quantity other than merely a parameter count. \cite{vapnik1971uniform} links uniform convergence of empirical risk to population risk with VC-dimension, serving as a capacity measure of a hypothesis class. \cite{bartlett2002rademacher} introduce Rademacher and Gaussian complexities as data-dependent capacity measures. \cite{bousquet2002stability} derive generalization bounds from algorithmic stability and show that many algorithms such as support vector machines, for regression and classification, satisfy the required stability conditions. \cite{mcallester1999pac} gives bounds in terms of posterior's discrepancy from a fixed prior using PAC-Bayesian framework. \cite{xu2017information} use an information--theoretic view for generalization in generic statistical learning problems. \cite{raginsky2026separating} study generalization from a deterministic viewpoint by decoupling the dynamics with usual probabilistic assumptions on data.

For neural networks, norm-based analyses \cite{bartlett2017spectrally,neyshabur2015norm} bound this gap using the magnitudes of the trained weights, scaling with the product of the layers' spectral norms. \cite{bartlett1998sample} finds that norm-based capacities describe the generalization performance of a neural network better than parameter count-based measures for pattern classification problems. \cite{bartlett2019nearly} calculate almost tight VC-dimension for piecewise linear neural networks with ReLU activation function. \cite{golowich2018size} find bounds for the Rademacher complexity of neural networks that are independent of network size. \cite{neyshabur2017pac} merge norm-based and PAC-Bayesian techniques to obtain a generalization bound with perturbation arguments. \cite{zhang2017understanding} experimentally show that standard deep neural networks trained via stochastic gradient descent can fit random labels while maintaining generalization well enough to fit the original labels. For Transformers specifically, \cite{trauger2024sequence} provide norm-based generalization bounds that are independent of input sequence length. \cite{edelman2022inductive} give covering-number-based arguments for capacity of the self-attention mechanism, which is a central part of Transformer architecture. \cite{li2026sharper} derive generalization bounds for Transformers also using covering numbers.

Our approach differs by depending on the geometry of data-law and quantizing the distributional error via Wasserstein distance on the true law and the sampled law. The final generalization error involves an approximation error due to the training mechanism \cite[Algorithm 1]{akman2026optimal} in addition to the distributional error.
\smallskip

\noindent{\bf Distributionally robust optimization:} Wasserstein distributionally robust optimization (WDRO) replaces the empirical risk with its worst-case counterpart over an ambiguity set of probability laws around the empirical measures as a ball in Wasserstein metric. Foundational results provide finite-sample performance guarantees by reformulating WDRO problems as tractable linear programs by \cite{mohajerin2018data}. \cite{blanchet2019quantifying} establish the duality theory for WDRO. The connection between WDRO and explicit regularization has been clarified by \cite{gao2023distributionally,shafieezadeh2019regularization}. \cite{lee2018minimax} proves generalization bounds for generic statistical learning using covering numbers from WDRO lens. \cite{gao2023finite} provides non-asymptotic guarantees using WDRO that avoid the curse of dimensionality.
\smallskip

\noindent{\bf Control---theoretic perspectives on Transformers:} \cite{weinan2017proposal} is one of the pioneers suggesting a control--theoretic view of neural networks. \cite{haber2018stable} develop stability and well-posedness results. \cite{han2019mean} follow an ensemble control approach to train deep neural networks. \cite{li2018maximum} reformulate deep learning as an optimal control problem, and the resulting problems are solved using Pontryagin's maximum principle, see \cite{liu2019deep} for a further survey and the references therein. \cite{geshkovski2025mathematical} study Transformers with a rigorous and mathematical analysis, and they further establish the clustering behaviors of Transformers \cite{geshkovski2023emergence}. The control perspective posed as an open question by \cite{geshkovski2025mathematical} is studied by \cite{akman2026optimal} by providing a discrete-time measure-valued formulation of Transformers trained using an approximate scheme through quantization and dynamic programming principle. Closest work to ours is \cite{kan2026optimal}, in which they cast Transformer training as a continuous-time optimal control problem with an optimal-transport regularization penalty. They derive a Lipschitz estimate for the pushforward by Transformer dynamics on data distributions with respect to Wasserstein distance. Also, using the convergence rate for Wasserstein distance in \cite{fournier2015rate}, they derive statistical generalization bounds as well. Their approach mostly relies on convexity for encoder-only/decoder-only models and a threshold hypothesis for the optimal transport regularization term, while we differ: Our stability estimate is for the value function and needs only compactness; the usual training paradigm with gradient descent family is replaced with dynamic programming; and the adversary perturbs the data-generating law on input-output pairs, rather than the input and output measures

\subsection{Our contributions} Classical generalization theory treats a Transformer as a member of a hypothesis class, and measures the generalization capacity through norms or covering numbers, not emphasizing a structural feature of the Transformer architecture: a token's evolution depends on the empirical measure of the sequence it belongs to. We take a different route, by regarding Transformers as controlled measure-valued dynamical systems whose controls are the layer weights, and ask how the learned controls perform when the associated empirical law is replaced with the true data-generating law. Within this formulation, we make two contributions.
\smallskip

\begin{itemize}
    \item[(I)] {\bf An explicit finite-sample guarantee for an implementable, fully-discretized optimal control problem.} Building on the doubly-lifted model of \cite{akman2026optimal}, we quantize the state, action, and measure-state spaces, yielding a finite Markov decision process whose optimal control policies are realizable through dynamic programming. For this training scheme, we prove an excess-risk bound (Theorem~\ref{thm:quantized}) that is explicit in sample size, and we transfer it to the base model at the cost of an explicit approximation error due to quantization (Theorem~\ref{thm:main}).
    \item[(II)] {\bf Wasserstein distributionally robust control for Transformers.} The doubly-lifted formulation described in Section~\ref{sec:formulation} enables another perspective for the generalization problem. By trying to find the optimal weight trajectory against the effort of an adversary choosing a perturbed initial law from a Wasserstein-ball around the true law, we equivalently formulate the generalization as a distributionally robust control problem where the geometry is provided by the Wasserstein metric. This alternative view establishes:
    \begin{itemize}
        \item For the described minimax problem, we prove the existence of a robust minimizer in Proposition~\ref{prop:existence_of_robust_min}, around which the value function performs well-enough, where the discrepancy is a controllable quantity.
        \item In Theorem~\ref{thm:robust-convergence}, we explain the relation between robust minimizers and training minimizers through (i) $\Gamma$-convergence of robust objective to training objective, and (ii) Painlev\'e-Kuratowski upper limit of the set of robust minimizers to the set of training minimizers.
        \item We prove that robust minimizers share the same generalization bounds with the training minimizers in Theorem~\ref{thm:r-quantized} in the quantized regime, then in Theorem~\ref{thm:r-main} for the base problem.
    \end{itemize}
\end{itemize}
All of these results are established via Lipschitz estimates provided in Lemmata~\ref{lem:lipschitz} and \ref{lem:lipU}.
\begin{figure}[H]
\centering
\begin{tikzpicture}[
    scale=0.75,
    transform shape,
    level/.style={draw, rounded corners, align=center,
                    minimum width=4.2cm, minimum height=1.1cm,
                    font=\small},
    arrow/.style={-{Latex[length=2.5mm]}, thick},
    label/.style={font=\footnotesize\itshape, align=left,
                    text width=5cm}
]

\node[level, fill=gray!5]  (base)   at (0,0)
  {Particle level\\[2pt] $X_t^i = (p_i, x^i_t) \in \underbrace{\mathrm{PE}_N\times\mathrm{S}}_{\XX}$\\$U_t := (W_t,A_t,b_t,Q_t,K_t,V_t)$};
\node[level, fill=blue!5]  (lift1)  at (0,-3)
  {Singly lifted\\[2pt] $\mu_t \in \cP(\XX)$\\
  $\mu_t := \frac{1}{N}\sum_{i=1}^N\delta_{X_t^{i}}$\\
  $\Phi$: Transformer flow map};
\node[level, fill=blue!10] (lift2)  at (0,-6)
  {Doubly lifted\\[2pt] $\PP \in \cP(\cZ)$,
   $\cZ=\cP(\XX)^2$\\
   $\PP$ population law};
\node[level, fill=blue!20] (quant) at (0,-9)
  {Triply quantized\\[2pt]
   $\hat\PP^{\ell,n} \in \cP(\cZ^{\ell,n})$,\\[0.2em]
   $\underline{U}^{(m)} \in \UU_m^T$,\\[0.2em]
   $\cZ^{\ell,n} = \cP^{(\ell)}(\XX_n)^2$};

\draw[arrow] (base)  -- (lift1);
\draw[arrow] (lift1) -- (lift2);
\draw[arrow] (lift2) -- (quant);

\node[label, anchor=west] at (2.6,-1.5)
  {Lift to measures:\\ restores Markov property};
\node[label, anchor=west] at (2.6,-4.5)
  {Lift to laws on data sets};
\node[label, anchor=west] at (2.6,-7.5)
  {Quantize $\XX$ (state space with positional encodings incorporated),\\ $\UU$ (weight/action space for one layer),\\
  $\cP(\XX)$ (measure-state space lifted from particles)};

\end{tikzpicture}
\caption{The hierarchy of models. Each downward arrow represents a structural operation: the first two restore Markovian property and enable a robust formulation, respectively. The third enables us to find concentration inequalities, and consequently, finite-sample generalization bounds. Finally, Theorem~\ref{thm:main} transfers bounds from the bottom level back to the top.}
\label{fig:hierarchy}
\end{figure}
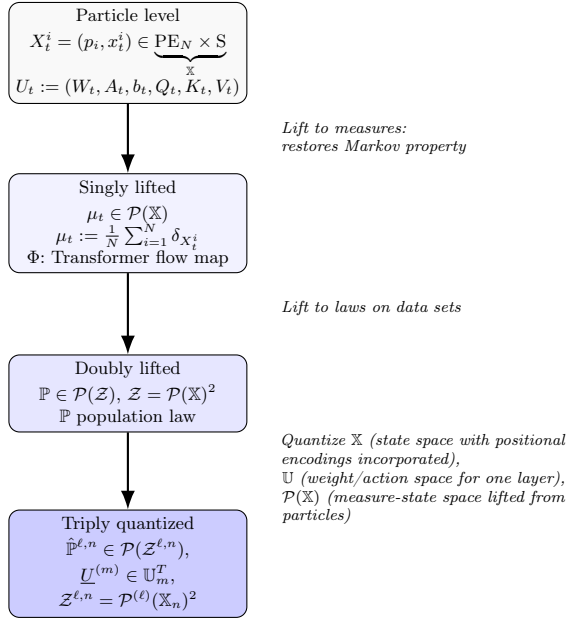

\subsection{Notation} The list of symbols used to describe the particle-level dynamics of Transformers.
\begin{itemize}
    \item $x^i$: feature vector for the $i^{\rm th}$ token
    \item $y^i$: label vector for the $i^{\rm th}$ token
    \item $W_t$, $A_t$, $b_t$: weights for the $t^{\text{th}}$ feed-forward layer
    \item $Q_t$, $K_t$, $V_t$: key, query and value matrices for the $t^{\text{th}}$ self-attention layer
    \item $T$: number of layers
    \item $\cT := \{0,1,\ldots,T-1\}$: layer indexing set (time-horizon)
    \item $\overline{\cT} := \cT\cup\{T\}$: layer indexing set including the terminal layer
    \item $\mathrm{S}$: state space of tokens
    \item $\mathrm{PE}_N$: positional encoding set
    \item $\cD$: generic data set
    \item $x_t^{i,k}$: tokens passing through the $t^{\text{th}}$ layer, placed at the $i^{\text{th}}$ position of the $k^{\text{th}}$ sequence in data set
\end{itemize}

\section{The Doubly Lifted Model of Transformers}\label{sec:formulation}
In this section, we describe the doubly lifted model introduced in \cite[Section 5]{akman2026optimal}. Transformers can be modeled at three different levels with increasing generality. In the base case, the particle-wise interactions can be described in continuous-time \cite[Equation 2.3]{geshkovski2025mathematical} or in discrete-time \cite[Section 2]{akman2026optimal}. However, such representation does not enjoy the Markovian property since the distribution of each particle depends on the mean-field term. Using measure-space lifting, we recover the Markovian property for empirical measure-representation for any ensemble of particles given as inputs to a Transformer. To investigate robustness properties of Transformers, it is easier to work with doubly lifted model, which is to be specified in this part.

\subsection{Dataset induced particle-level dynamics}
Before describing the model, as a reference, we here give a legend of the notation and conventions used throughout in this paper.

Let $\mathrm{S} \subset \RR^d$ be the particles' state space for some $d\in\NN$. We work with Transformers of fixed input sequence length $N\in\NN$, and of number of layers $T\in\NN$. The indexing set of particles is denoted by $\cN$, {\rm i.e.} $\cN := \{1,\ldots, N\}$. Moreover, $\cT = \{0,\ldots, T-1\}$ and $\overline{\cT} = \cT\cup\{T\}$ are used heavily in this paper in order to denote the index set of layers except the terminal layer and all layers, respectively. For each $i\in\cN$, the particle $x^i = (x_t^i)_{t\in\overline{\cT}}$'s evolution is governed by
\begin{equation}\label{eq:dynamics}
    \begin{cases}
        \displaystyle x_{t+1}^i = W_t\,\sigma(A_tx_t^i + b_t) + \sum_{j=1}^N \frac{ \e^{\beta\inner{Q_tx_t^i}{K_tx_t^j}}}{\sum_{j^\prime=1}^N\e^{\beta\inner{Q_tx_t^i}{K_tx_t^{j^\prime}}}}V_tx_t^j & \text{where } t\in\cT \\
        \displaystyle x^i_0 = x^i \in \mathrm{S}.
    \end{cases}
\end{equation}

Here $\sigma:\RR^{d_1}\!\rightarrow\!\RR^{d_1}$, $d_1\in\NN$, is a $1$-Lipschitz function, usually referred to as an {\it activation function}. The weights $W_t,A_t,b_t,Q_t,K_t,V_t$ are accessible by entire ensemble at each time $t\in\cT$, and $W_t\in\RR^{d\times d_1}$, $A_t\in\RR^{d_1\times d}$, $b_t\in\RR^{d_1}$, $Q_t, K_t\in\RR^{d_2\times d}$, and $V_t\in\RR^{d\times d}$ for some $d_1,d_2\in\NN$. For brevity, $U_t = (W_t,A_t,b_t,Q_t,K_t,V_t)$.

We also incorporate positional encodings to this system for preventing any loss of positional information when the measure lifting is applied. Denoting the positional encodings' set by $\mathrm{PE}_N$, we associate particle $x^i$ with the label $i/N$ for each $i\in\cN$. In other words, $\mathrm{PE}_N := \{i/N : i\in\cN\}$. From here, the augmented dynamics can be specified as follows:
\begin{equation}\label{eq:augdynamics}
    \begin{aligned}
        &X_{t+1}^i := (p_i,x_{t+1}^i) = \left(p_i,W_t\,\sigma(A_tx_t^i + b_t) + \sum_{j=1}^N \frac{ \e^{\beta\inner{Q_tx_t^i}{K_tx_t^j}}}{\sum_{j^\prime=1}^N\e^{\beta\inner{Q_tx_t^i}{K_tx_t^{j^\prime}}}}V_tx_t^j\right)\\
        &X_0^i = (p_i,x^i_0) = (p_i,x^i) \in \mathrm{PE}_N\times\mathrm{S}
    \end{aligned}
\end{equation}
We define the state space of augmented particles by $\XX := \mathrm{PE}_N\times\mathrm{S}$.

\subsection{Lifted dynamics} Let $\cD := \{(\mathbf{x}^k, \mathbf{y}^k): k=1,2,\ldots,K\}\subseteq\mathrm{S}^N\times\mathrm{S}^N$ be a data set, for some $K\in\NN$. We denote a generic feature vector by $\mathbf{x}^k := (x^{1,k},\ldots,x^{N,k})$ and $\mathbf{y}^k := (y^{1,k},\ldots,y^{N,k})$. These ensembles of particles give rise to empirical measures
\[
    \mu_0^k := \frac{1}{N}\sum_{i=1}^N\delta_{(p_i,x^{i,k})} \text{ and } \nu^k = \frac{1}{N}\sum_{i=1}^N\delta_{(p_i,y^{i,k})}.
\]
Then $\mu^k = (\mu_t^k)_{t\in\overline{\cT}}$ denotes the measure process where the underlying particles evolve with \eqref{eq:augdynamics}. It is possible to formulate this evolution at the measure-level. Before that we define the function $f:\XX\times\UU\times\cP(\XX)\to\XX$ for all $(p,x)\in\XX$, $U = (W, A, b,$ $ Q, K, V) \in\UU$, and $\mu\in\cP(\XX)$ by
\begin{equation}\label{eq:McKean---Vlasov}
    f((p, x), U, \mu) := \left(p,  W\sigma(Ax + b) + \dfrac{1}{\int_\mathrm{S} \mu_\mathrm{S}(\de z)\,\e^{\beta\inner{Qx}{Kz}}}\int_\mathrm{S} \mu_\mathrm{S}(\de \tilde{z})\,\e^{\beta\inner{Qx}{K\tilde{z}}}\,V\tilde{z}\right),
\end{equation}
which models the flow of a single particle. Here, $\mu_{\mathrm{S}}$ denotes the marginal on $\mathrm{S}$. In particular, taking $(p,x) = (p_i, x_t^{i,k}) = X_t^{i,k}$, $U_t\in\UU$ as in \eqref{eq:augdynamics}, and $\mu_t^k = \frac{1}{N}\sum_{i=1}^N\delta_{(p_i,x_t^{i,k})}$ for all $i\in\cN$, $k\in\cK$, and $t\in\cT$, $f$ satisfies
\[
    f(X_t^{i,k}, U_t, \mu_t^k) = X_{t+1}^{i,k}.
\]
Then the measure evolution is given by the map $\Phi:\cP(\XX)\times\UU\to\cP(\XX)$ for all $\mu\in\cP(\XX)$ and $U\in\UU$ by calculating the following push-forward measure
\begin{equation}\label{eq:measure-evolution}
    \Phi(\mu, U) := \mu\circ T_{\mu, U}^{-1}
\end{equation}
where $T_{\mu, U}:\XX\to\XX$ is given by $T_{\mu, U}(X) := f(X,U,\mu)$ for all $X\in\XX$. By a change of variables argument, it is straightforward to obtain the relation $\mu_{t+1}^k = \Phi(\mu_t^k, U)$ for any $t\in\cT$, $k=1,\ldots,K$ and $U\in\UU$. Such relation describes a discrete-time Fokker--Planck type of evolution. A multi-variable generalization of $\Phi$ as $\boldsymbol{\Phi}:\cP(\XX)^K\times\UU\to\cP(\XX)^K$ is given, for all $(\mu^1,\ldots,\mu^K)\in\cP(\XX)^K$ and $U\in\UU$, by
\begin{equation}\label{eq:ensemble-dynamics}
    \boldsymbol{\Phi}(\mu^1,\ldots,\mu^K; U) = (\Phi(\mu^1, U), \ldots, \Phi(\mu^K, U)).
\end{equation}

For ease, for all $t\in\overline{\cT}$, we also define
\begin{equation}\label{eq:t-fold-map}
    \Phi_t(\mu, \underline{U}_{0:t-1}) := \Phi(\Phi(\cdots\Phi(\Phi(\mu, U_0), U_1), \cdots),U_{t-1})
\end{equation}
which is the $t$-fold application of $\Phi$ under a weight profile $\underline{U}_{0:t-1}\in\UU^t$.

\ka{The training procedure of the lifted model is given in terms of dynamic programming recursions:
\begin{equation}\label{eq:dp}
    \begin{cases}
        \displaystyle\fC_{T}(\mu_T^{1},\ldots,\mu_T^{K}) := \frac{1}{K}\sum_{k=1}^K W_{2,\lambda}(\mu_T^k, \nu^k) \\
        \displaystyle\fC_{t}(\mu_t^{1},\ldots,\mu_t^{K}) := \inf_{U\in\UU}\fC_{t+1}(\boldsymbol{\Phi}(\mu_t^{1},\ldots,\mu_t^{K}; U))
    \end{cases}
\end{equation}
with $\fC_0(\mu_0^1,\ldots,\mu_0^K)$ is the optimal training error. Under mild assumptions, the existence of an optimal closed-loop policy is proven in \cite[Theorem~9]{akman2026optimal} by means of showing that the measurable selection conditions hold (see \cite[Theorem~3.2.1 and §3.3]{hernandez2012discrete}). Since the initial data is fixed before the start of training procedure, the optimal closed-loop policy can be transformed into genuine open-loop policy due to the deterministic nature of Transformer dynamics. The resulting open-loop policies are compatible with the operational properties of Transformers, see \cite[§3.4]{akman2026optimal}.}

\subsection{Doubly lifted model} Given a sequence of data sets $\cD_r := \{(\boldsymbol{x}^{k,r}, \boldsymbol{y}^{k,r}) : k=1,2,\ldots,K_r\}$, where $r\in\NN$, we lift them to the empirical measures of the following form
\[
    \PP_r := \frac{1}{K_r}\sum_{k=1}^{K_r}\delta_{(\mu_0^{k,r}, \nu^{k,r})},
\]
where $\mu_0^{k,r} := \frac{1}{N}\sum_{i=1}^N\delta_{(p_i,x^{i,k,r})}$ and $\nu^{k,r} := \frac{1}{N}\sum_{i=1}^N\delta_{(p_i,y^{i,k,r})}$ for all $r\in\NN$ and all $k\in\{1,\ldots,K_r\} =: \cK_r$. Writing $\cZ := \cP(\XX)\times\cP(\XX)$, we see $\PP_r \in \cP(\cZ)$. Now, we define the cost functional (with an abuse of notation), for all $\underline{U}\in\UU^T$ and $P\in \cP(\cZ)$, as
\[
    V(P, \underline{U}) := \int_{\cZ}P(\de\mu,\de\nu)\,W_{2,\lambda}\bigl(\Phi_T(\mu, \underline{U}), \nu\bigr)^2,
\]
Here we use the notation $\underline{U} = (U_0,U_1,\ldots,U_{T-1})$ $\in\UU^T$ and $\underline{U}|_t = (U_0,\ldots,U_t)$ for all $t\in\overline{\cT}$. In short, $V$ is the value function for the empirical measure of a given data set that passes the first marginal through the Transformer flow using a trajectory of control actions, and compares the resulting distribution with the second marginal. The cost functional here is defined by $\fc((\mu,\nu), \underline{U}^T) := W_{2,\lambda}(\Phi_T(\mu, \underline{U}), \nu)^2$ for all $(\mu,\nu) \in \cZ$ and all $\underline{U}\in\UU^T$. For an empirical measure, for example considering $\PP_r$, the value function collapses to
\[
    V(\PP_r, \underline{U}) := \frac{1}{K_r}\sum_{k=1}^{K_r} W_{2,\lambda}\bigl(\Phi_T(\mu_0^{k,r}, \underline{U}^T), \nu^{k,r}\bigr)^2,
\]
\ka{which is equivalent to the dynamic programming based training procedure described in \eqref{eq:dp}.}

\begin{assumption}
\label{assumption:compactness}
    The particle state space $\mathrm{S}$ and the action space $\UU$ are compact. Moreover, Transformer dynamics are forward-invariant, {\rm i.e.} if $\mathsf{T}_{\underline{U}}$ is a Transformer constructed via the weights $\underline{U}\in\UU^T$, then $\mathsf{T}_{\underline{U}}({\rm S}^N)\subseteq{\rm S}^N$.
\end{assumption}

In applications, the compactness of the state space can be substituted with projections onto a compact and convex set, which is more common. Alternatively, the compactness of $\UU$ is replaced with a regularity term on the cost. However, these do not add significant deviations in the analysis and for the simplicity of presentation, we impose compactness on the action space.

\begin{remark}
The topologies on $\cP(\XX)$ and $\cP(\cZ)$ play an important role in our analysis. We endow $\cP(\XX)$ with the weak$^*$ topology induced by the space of continuous, bounded, and real-valued functions on $\XX$, denoted $\cC_b(\XX)$. So, $\mu_n\xrightarrow{\mathrm{w}^*} \mu$ in $\cP(\XX)$ if, for all $g\in\cC_b(\XX)$, we have
\[
    \int_\XX g\,\de\mu_n \longrightarrow \int_\XX g\,\de\mu.
\]
Under Assumption~\ref{assumption:compactness}, $(\cP(\XX), \mathrm{w}^*)$ is metrizable by $p$-Wasserstein family, and even better, is compact. Choosing a {\it positional enforcement constant} $\lambda>0$, see \cite[Section 3]{akman2026optimal}, we write the correct analogue of $p$-Wasserstein distance for $\cP(\XX)$: For all $\mu,\nu\in\cP(\XX)$
\begin{equation}\label{eq:was-2}
    W_{2,\lambda}(\mu,\nu) := \left(  \inf_{\pi\in\mathrm{C}(\mu,\nu)}\int_{\XX\times\XX}\pi(\de p_x, \de x; \de p_y, \de y)\,c_\lambda(p_x, x; p_y, y)^2 \right)^{\frac{1}{2}}
\end{equation}
for all $\mu,\nu\in\cP(\XX)$, where $\mathrm{C}(\mu,\nu) := \{\pi\in\cP(\XX\times\XX) : \pi(\cdot \times \XX) = \mu(\cdot), \pi(\XX\times\cdot) = \nu(\cdot)\}$ and $c_\lambda(p_x, x; p_y, y) := \sqrt{\|x - y\|_{\mathrm{S}}^2 + \lambda|p_x - p_y|^2}$. For large enough $\lambda>0$, the changes in the term $\lambda|p_x - p_y|^2$ dominates the changes in $\|x - y\|_{\mathrm{S}}^2$. Therefore, the minimization problem in \eqref{eq:was-2} prioritizes to minimize $\lambda|p_x - p_y|^2$. As a consequence, $W_{2,\lambda}$ matches the positional encodings first.

Since $\cP(\XX)$ is compact under Assumption~\ref{assumption:compactness}, $\cP(\cZ)$ is compact as well in the product weak$^*$ topology, for which we use the metric
\begin{equation}
    \cW_1(P, Q) := \inf_{\Pi\in\mathrm{C}(P,Q)}\int_{\cZ\times\cZ}\!\Pi(\de(\mu_P,\nu_P), \de(\mu_Q,\nu_Q))\bigg(W_{2,\lambda}(\mu_P, \mu_Q) + W_{2,\lambda}(\nu_P, \nu_Q)\bigg)
\end{equation}
for all $P,Q\in\cP(\cZ)$, which generates the same topology. One other reason for the choice of $\cW_1$ is due to two reasons: (i) $\cW_1$ metrizes the weak$^*$-topology on $\cP(\cZ)$ \cite[Theorem~6.9]{villani2009optimal}, (ii) $\cW_1$ has an alternative representation (Kantorovich---Rubinstein duality \cite[Remark~6.5]{villani2009optimal}), which is utilized in Lemma~\ref{lem:lipschitz}.
\end{remark}

In light of \cite[Theorem 16]{akman2026optimal}, we know that as the empirical measures raised by the respective datasets converge to the true distribution, the optimal value function converges as well. In this section, we derive a Lipschitz-type estimates for the value functions so that we can quantitatively characterize the generalization behavior. In our results in this section, we utilize the following result.

\begin{lemma}{\cite[Lemma~15]{akman2026optimal}}\label{lem:value-function-is-continuous}
    Under Assumption~\ref{assumption:compactness}, the value function $V$ is jointly continuous on $\cP(\cZ)\times\UU^T$.
\end{lemma}

Later in this paper, we develop qualitative bounds for generalization error. To achieve those, we need some Lipschitz-type estimates for the value function. The subsequent result requires the following estimate.

\begin{lemma}{\cite[Corollary~19]{akman2026optimal}}\label{cor:phi-lipschitz}
    For any $\boldsymbol{\mu},\boldsymbol{\nu}\in\cP(\XX)^K$ and $U\in\UU$, we have
    \begin{equation}\label{eq:estimate-1}
        \sum_{k=1}^K W_{2,\lambda}(\Phi(\mu^k,U), \Phi(\nu^k, U)) \leq L_\Phi \sum_{k=1}^KW_{2,\lambda}(\mu^k,\nu^k)
    \end{equation}
    for some $L_\Phi \geq 0$. We here write $\boldsymbol{\mu} = (\mu^1,\ldots,\mu^K)$ and $\boldsymbol{\nu} = (\nu^1,\ldots,\nu^K)$.
\end{lemma}

In particular, \eqref{eq:estimate-1} holds when $K=1$ as the proof of it essentially establishes the Lipschitz stability of the map $\Phi$.

\begin{lemma}\label{lem:lipschitz}
    There exists $L > 0$ such that, for all $P,Q\in\cP(\cZ)$, we have $|V(P,\underline{U}) - V(Q, \underline{U})| \leq L \cW_1(P,Q)$ uniformly in $\underline{U}\in\UU^T$.
\end{lemma}
\begin{proof}
    Fix $\underline{U}\in\UU^T$, then for all $\mu_1,\mu_2\in\cP(\XX)$, note that we have
    \[
        W_{2,\lambda}\bigl(\Phi_T(\mu_1,\underline{U}), \Phi_T(\mu_2,\underline{U})\bigr) \leq L_{\Phi}^TW_{2,\lambda}(\mu_1,\mu_2)
    \]
    by \cite[Corollary 19]{akman2026optimal}. As a second note, for all $\mu_1,\nu_1,\mu_2,\nu_2\in\cP(\XX)$, observe
    \begin{align*}
        \bigl|W_{2,\lambda}(\Phi_T(\mu_1,\underline{U}),\nu_1)^2- W_{2,\lambda}(\Phi_T(\mu_2,\underline{U}),\nu_2)^2\bigr| &\leq B_W\cdot |W_{2,\lambda}(\Phi_T(\mu_1,\underline{U}),\nu_1) \\
        &\hspace{4em}- W_{2,\lambda}(\Phi_T(\mu_2,\underline{U}),\nu_2)|
    \end{align*}
    Here, the basic identity $|a^2-b^2| = |a-b||a+b|$ is utilized. Let $B_W$ be the bound dominating the term $|W_{2,\lambda}(\Phi_T(\mu_1,\underline{U}),\nu_1) + W_{2,\lambda}(\Phi_T(\mu_2,\underline{U}),\nu_2)|$ due to compactness of $\cP(\XX)$ and the continuity of the metric, and $\Phi$.  For $|W_{2,\lambda}(\Phi_T(\mu_1,\underline{U}),\nu_1)- W_{2,\lambda}(\Phi_T(\mu_2,\underline{U}),\nu_2)|$, we decompose as
    \begin{align*}
        |W_{2,\lambda}(\Phi_T(\mu_1,\underline{U}),\nu_1)- W_{2,\lambda}(\Phi_T(\mu_2,\underline{U}),\nu_2)| &\leq |W_{2,\lambda}(\Phi_T(\mu_1,\underline{U}),\nu_1)\\
        &\hspace{3.5em}- W_{2,\lambda}(\Phi_T(\mu_1,\underline{U}),\nu_2)|\\
        &\quad+|W_{2,\lambda}(\Phi_T(\mu_1,\underline{U}),\nu_2)\\
        &\hspace{3.5em}- W_{2,\lambda}(\Phi_T(\mu_2,\underline{U}),\nu_2)|\\
        &\leq W_{2,\lambda}(\nu_1,\nu_2)\\
        &\hspace{3.5em}+ W_{2,\lambda}(\Phi_T(\mu_1,\underline{U}),\Phi_T(\mu_2,\underline{U}))
        \intertext{since $\Phi_T$ is Lipschitz, we have}
        &\leq W_{2,\lambda}(\nu_1,\nu_2) + L_{\Phi}^TW_{2,\lambda}(\mu_1,\mu_2)\\
        &\leq\max\{1,L_{\Phi}^T\}(W_{2,\lambda}(\mu_1,\mu_2)\\
        &\hspace{7em}+ W_{2,\lambda}(\nu_1,\nu_2))
    \end{align*}

    Getting back to the claim, the difference between value functions realized with different laws, but with same control trajectories is bounded as follows
    \begin{align*}
        |V(P,\underline{U}) - V(Q, \underline{U})| &= \bigg|\int P(\de\mu,\de\nu)\,W_{2,\lambda}(\Phi_T(\mu,\underline{U}),\nu)^2\\
        &\hspace{5em}- \int Q(\de\mu,\de\nu)\,W_{2,\lambda}(\Phi_T(\mu,\underline{U}),\nu)^2\bigg|\\
        &=\bigg|\int (P-Q)(\de\mu,\de\nu)\,W_{2,\lambda}(\Phi_T(\mu,\underline{U}),\nu)^2\bigg|\\
        &\leq B_W\max\{1,L_{\Phi}^T\}\cW_1(P,Q) =: L\cW_1(P,Q).
    \end{align*}
    Since the cost function is Lipschitz, the inequality in the last line is due to Kantorovich---Rubinstein duality, see Remark~6.5 in \cite{villani2009optimal}.
\end{proof}

\section{Quantitative Results for the Triply Quantized Model}\label{sec:quantized}
This section presents quantitative generalization bounds for the quantized training scheme. We first adapt the quantized model in \cite{akman2026optimal} to the doubly lifted model used in this paper by quantizing state, action and measure-state spaces. Then, we prove a concentration inequality for the empirical quantized double lifted law. Finally, we derive explicit generalization bounds at the quantized level.
\subsection{Adaptation of triply quantized setting for doubly lifted model} We merge the settings of Sections 4 and 5 in \cite{akman2026optimal}. Under Assumption~\ref{assumption:compactness}, due to the total boundedness of $\mathrm{S}$, for any $n\in\NN$, there is a finite set $\{s^{(n), j}: j\in J(n)\} \subseteq \mathrm{S}$ such that
\[
    \min_{j\in J(n)} \|x - s^{(n),j}\| < \frac{1}{n}
\]
for all $x\in\mathrm{S}$. We denote the set of quantized states by $\mathrm{S}_n := \{s^{(n), j}: j\in J(n)\}$ for all $n\in\NN$. The positional encodings are incorporated into state space, that is $\XX_n := \mathrm{PE}_N \times \mathrm{S}_n$. We utilize the {\it position-sensitive hard quantizer} $Q_n: \XX \to \XX_n$ as follows
\begin{equation}
    Q_n(p,x) := (p, \operatorname{argmin}_{s\in\mathrm{S}_n}\|x-s\|), \quad (p,x)\in\XX,
\end{equation}
where the ties are broken so that the measurability is preserved.

We also quantize the action space $\UU$ with respect to the norm $\|U\|_\UU^2 = \Tr(WW^T) + \Tr(AA^T) + \Tr(bb^T) + \Tr(QQ^T) + \Tr(KK^T) + \Tr(VV^T)$ for all $U=(W,A,b,Q,K,V)\in\UU$. By Assumption~\ref{assumption:compactness}, we also know that $\UU$ is compact, as a result, for any $m\in\NN$, there is a finite set $\UU_m := \{U^1, \ldots, U^{(M(m))}\} \subseteq \UU$ such that
\begin{equation}
    \min_{j=1,\ldots,M(m)}\|U^{(m),j}-U\|_\UU \leq \frac{1}{m}
\end{equation}
for all $U\in\UU$.

We see the resulting space as follows
\begin{equation}
\cP(\XX_n) = \{p = (p_a)_{a=1}^{|\XX_n|}\in\RR^{|\XX_n|} : p_a \geq 0 \text{ and } \sum_{a=1}^{|\XX_n|} p_a = 1\}.
\end{equation}
On the measure-level, this problem is finite-dimensional, but lacking compactness. For computational feasibility, we still need to quantize the measure as well. In this regard, we use the algorithm in \cite{reznik2011algorithm} and the \emph{reconstruction points} of $\cP(\XX_n)$ as the set
\begin{equation}
    \cP^{(\ell)}(\XX_n) := \big\{(p_a)_{a=1}^{|\XX_n|}: p_a = \frac{r_a}{\ell}, r_a \in \NN_0, \sum_{a=1}^{|\XX_n|}r_a = \ell\big\},
\end{equation}
which is a set of cardinality $\binom{\ell + |\XX_n| - 1}{|\XX_n| - 1}$. Let $R_\ell: \cP(\XX_n)\to\cP^{(\ell)}(\XX_n)$ be the quantizer given by \cite[Algorithm 1]{reznik2011algorithm}. Then $R_\ell$ maps each $\hat{\mu}\in\cP(\XX_n)$ to the closest element in $\cP^{(\ell)}(\XX_n)$ in the sense of $W_{2,\lambda}$ metric.

Finally, we define $\cZ^{\ell,n} := \cP^{(\ell)}(\XX_n)\times\cP^{(\ell)}(\XX_n)$ as the base space for the doubly lifted model. Then the measures we obtain live in $\cP(\cZ^{\ell,n})$. Let $\hat{\PP}_r$ be the empirical measure raised by the dataset $\cD_r := \{(\boldsymbol{x}^{k,r}, \boldsymbol{y}^{k,r}) : k=1,2,\ldots,K_r\}$ as follows:
\[
    \hat\PP_r^{\ell,n} := \frac{1}{K_r}\sum_{k=1}^{K_r}\delta_{\bigl(R_\ell(\mu^k\circ Q_n^{-1}),\;R_\ell(\nu^k\circ Q_n^{-1})\bigr)}.
\]
The weak$^*$ topology of $\cP(\cZ^{\ell,n})$ can be characterized by $W_1$ distance, which is defined for all $\hat P,\hat Q\in\cP(\cZ^{\ell,n})$ by
\[
    \cW_1^{\ell,n}(\hat P,\hat Q) := \inf_{\hat{\pi}\in\hat\Pi(\hat P,\hat Q)}\int_{\cZ^{\ell,n}}\hat\pi\bigl(\de(\hat\mu_1, \hat\nu_1), \de(\hat\mu_2,\hat\nu_2)\bigr)\,d_1^{\,\ell,n}((\hat\mu_1, \hat\nu_1), (\hat\mu_2,\hat\nu_2)),
\]
induced by the metric $d_1^{\ell,n}\bigl((\hat\mu_1, \hat\nu_1), (\hat\mu_2,\hat\nu_2)\bigr) = W_{2,\lambda}(\hat\mu_1,\hat\mu_2) + W_{2,\lambda}(\hat\nu_1,\hat\nu_2)$ on $\cZ^{\ell,n}$. The set $\hat\Pi(\hat P,\hat Q)$ stands for all couplings of $\hat P$ and $\hat Q$. The flow map for the quantized model is given by
\begin{equation}\label{eq:quantized-flow}
    \Phi^{\ell,n}(\hat\mu, U^{(m)}) := R_\ell(\Phi(\hat{\mu}\circ\iota^{-1}, \underline{U}^{(m)})), \quad\forall\hat\mu\in\cP^{(\ell)}(\XX_n),\;\forall U^{(m)}\in\UU_m.
\end{equation}
Note that since $\UU_m\subset\UU$, this is a well-defined function. Also, $T$-fold application of it is defined similar to \eqref{eq:t-fold-map}.

\subsection{Value function and robust objective} For any $\hat P\in\cP(\cZ^{\ell,n})$ and any $\underline{U}^{(m)}\in\UU_m^T$, we define the corresponding value function
\[
    V^{\ell,n}(\hat P, \underline{U}^{(m)}) := \int_{\cZ^{\ell,n}}\!\!\!\hat P(\de\hat\mu,\de\hat\nu)\,W_{2,\lambda}(\Phi_{T}^{\ell,n}(\hat\mu,  \underline{U}^{(m)}), \hat\nu)^2.
\]

\ka{Let $\hat\PP_r^{\ell,n}$ be the empirical measure on pairs of empirical measures, raised to a probability measure by an i.i.d. sample of $K_r$ observations obtained from the true law, and quantized by $R_\ell$ and $Q_n$.} 

For this setup, we define the ambiguity set
\[
    A_\rho(\hat\PP_r^{\ell,n}) := \bigl\{\hat P\in\cP(\cZ^{\ell,n}) : \cW_1^{\ell,n}(\hat P, \hat\PP_r^{\ell,n}) \leq \rho\bigr\}
\]
for $\rho>0$. Our robust objective is now
\begin{equation}
    \inf_{\underline{U}^{(m)}\in\,\UU_m^T}\sup_{P\in A_\rho(\hat\PP_r^{\ell,n})}V^{\ell,n}(P,\underline{U}^{(m)}).
\end{equation}

\ka{We write 
\begin{equation}\label{eq:quantized-true}
    \hat\PP^{\ell,n}(\de\hat\mu, \de\hat\nu) := \PP(\de R_\ell(\mu\circ Q_n^{-1}), \de R_\ell(\nu\circ Q_n^{-1}))
\end{equation}
for the quantized version of the true law, from which we sample data for the following concentration inequality.}

\begin{proposition}\label{prop:concentration}
    \ka{Let the quantized samples
    \[
        (\hat\mu^{1,(\ell,n)},\hat\nu^{1,(\ell,n)}),(\hat\mu^{2,(\ell,n)}\hat\nu^{2,(\ell,n)}),\ldots,(\hat\mu^{K_r,(\ell,n)},\hat\nu^{K_r,(\ell,n)})
    \]
    be i.i.d. with law $\hat{\PP}^{\ell,n}$ for some $\ell,n\in\NN$. Let $\hat\PP_r^{\ell,n} := \frac{1}{K_r}\sum_{k=1}^{K_r}\delta_{(\hat\mu^{k,(\ell,n)},\hat\nu^{k,(\ell,n)})}$. Then, under Assumption~\ref{assumption:compactness}, we have}   
    \[
        \Prob\bigg\{\cW_1^{\ell,n}(\hat{\PP}_r^{\ell,n}, \hat{\PP}^{\ell,n}) - \EE\cW_1^{\ell,n}(\hat{\PP}_r^{\ell,n}, \hat{\PP}^{\ell,n}) > \epsilon\bigg\} \leq \exp\left(-\frac{2K_r\epsilon^2}{\operatorname{diam}(\cZ^{\ell,n})^2}\right).
    \]
    for all $\epsilon > 0$.
\end{proposition}
\begin{proof}
Write $Z^k = (\hat\mu^{k,(\ell,n)},\hat\nu^{k,(\ell,n)})$ for all $k=1,\ldots,K_r$ and define
\[
    F(z^1,\ldots,z^{K_r}) := \cW_1^{\ell,n}\bigg(\frac{1}{K_r}\sum_{k=1}^{K_r} \delta_{z^k}, \hat\PP^{\ell,n}\bigg)
\]
for all $(z^1,\ldots,z^{K_r})\in(\cZ^{\ell,n})^{\times K_r}$. Then, $    F(Z^1,\ldots,Z^{K_r}) = \cW_1^{\ell,n}(\hat\PP_r^{\ell,n}, \hat\PP^{\ell,n})$.
Let $(z^1,\ldots,z^{K_r})\in(\cZ^{\ell,n})^{\times K_r}$ and $z^{i\prime}\in(\cZ^{\ell,n})$ for some $i=1,\ldots,K_r$ such that we construct
\[
    P_1 = \frac{1}{K_r}\sum_{k=1}^{K_r}\delta_{z^k} \text{ and } P_2 = \frac{1}{K_r}\sum_{k=1, k\neq i}^{K_r}\delta_{z^k} + \frac{1}{K_r}\delta_{z^{i\prime}}.
\]
By reverse triangle inequality for $\cW_1$, we get
\[
    |F(z^1,\ldots,z^{K_r}) - F(z^1,\ldots, z^{i\prime},\ldots,z^{K_r})| \leq \cW_1^{\ell,n}(P_1,P_2).
\]
As $P_1$ and $P_2$ differ by $\frac{1}{K_r}$ by moving from $z^i$ to $z^{i\prime}$, it follows that
\[
    |F(z^1,\ldots,z^{K_r}) - F(z^1,\ldots, z^{i\prime},\ldots,z^{K_r})| \leq \frac{\operatorname{diam}(\cZ^{\ell,n})}{K_r}.
\]
Using McDiarmid's inequality \cite[Lemma 1.2]{mcdiarmid1989method}, we get
\[
    \Prob\left\{ F - \EE F \geq \epsilon \right\} \leq \exp\left(-\frac{2\epsilon^2}{\sum_{i=1}^{K_r} \left(\frac{\operatorname{diam}(\cZ^{\ell,n})}{K_r}\right)^2}\right) = \exp\left(-\frac{2K_r\epsilon^2}{\operatorname{diam}(\cZ^{\ell,n})^2}\right),
\]
which yields the claim.
\end{proof}

The next result constitutes the backbone of Theorem~\ref{thm:quantized} and later Theorem~\ref{thm:main}.
\begin{lemma}\label{lem:event}
    Let the quantized samples
    \[
        (\hat\mu^{1,(\ell,n)},\hat\nu^{1,(\ell,n)}),(\hat\mu^{2,(\ell,n)}\hat\nu^{2,(\ell,n)}),\ldots,(\hat\mu^{K_r,(\ell,n)},\hat\nu^{K_r,(\ell,n)})
    \]
    be i.i.d. with law $\hat{\PP}^{\ell,n}$ for some $\ell,n\in\NN$. Let $\hat\PP_r^{\ell,n} := \frac{1}{K_r}\sum_{k=1}^{K_r}\delta_{(\hat\mu^{k,(\ell,n)},\hat\nu^{k,(\ell,n)})}$. For all $\delta\in(0,1)$, define 
    $$\rho_{r}^{\ell,n}(\delta) := \frac{D_{\ell,n}\sqrt{|\cZ_{\ell,n}|}}{2\sqrt{K_r}} + D_{\ell,n}\sqrt{\frac{\log(1/\delta)}{2K_r}}$$  
    where $D_{\ell,n} := \operatorname{diam}(\cZ_{\ell,n})$ for each $\ell,n\in\NN$.
    Then, under Assumption~\ref{assumption:compactness}, the event $\{\cW_1^{\ell,n}(\hat\PP_r^{\ell,n}, \hat\PP^{\ell,n}) \leq \rho_r^{\ell,n}(\delta)\}$ occurs with probability at least $1-\delta$. 
\end{lemma}
\begin{proof}
    We present the proof of this result in three separate pieces.
    
{\noindent\bf Step-1:} Fix an enumeration $\cZ_{\ell,n} := \{z^i : i\in\cM_{\ell,n}\}$ where $\cM_{\ell,n}:=\{1,\ldots,|\cZ_{\ell,n}|\}$. Let us write
\[
    \hat\PP_r^{\ell,n} = \sum_{i\in\cM_{\ell,n}}p_{i,r}\delta_{z^i} \text{ and } \hat\PP^{\ell,n} = \sum_{i\in\cM_{\ell,n}}p_{i}\delta_{z^i}.
\]
For $\hat\PP_r^{\ell,n}$, we have its exact representation in terms of samples; if $Z^k = (\hat\mu^{k,(\ell,n)},\hat\nu^{k,(\ell,n)})$ for all $k=1,\ldots,K_r$, then $\hat\PP_r^{\ell,n} = \frac{1}{K_r}\sum_{k=1}^{K_r}\delta_{Z^k}$. In this way, for each $i\in\cM_{\ell,n}$, $p_{i,r}$ can be seen as a sum Bernoulli random variables $B_k^{i} := \rchi_{z^i}(Z^k) \sim \mathrm{Bernoulli}(p_i)$ ($\rchi$ is the indicator function) as follows:
\[
    p_{i,r} = \frac{1}{K_r}\sum_{k=1}^{K_r}B_k^i.
\]

{\noindent\bf Step-2:} We wish to bound the expectation $\EE\cW_1(\hat\PP_r^{\ell,n},\hat\PP^{\ell,n})$ as follows
\ka{\begin{align*}
    \EE\cW_1(\hat\PP_r^{\ell,n},\hat\PP^{\ell,n}) &\leq D_{\ell,n}\EE\|\hat\PP_r^{\ell,n} - \hat\PP^{\ell,n}\|_{\rm TV}\leq\frac{D_{\ell,n}}{2}\sum_{i\in\cM_{\ell,n}}\EE|p_{i,r} - p_i|\\
    &\leq\frac{D_{\ell,n}}{2}\sum_{i\in\cM_{\ell,n}}\sqrt{\EE|p_{i,r} - p_i|^2}
    \intertext{by Cauchy-Bunyakovsky-Schwarz inequality}
    &=\frac{D_{\ell,n}}{2}\sum_{i\in\cM_{\ell,n}}\sqrt{{\rm Var}(p_{i,r})}\\
    &=\frac{D_{\ell,n}}{2}\sum_{i\in\cM_{\ell,n}}\sqrt{{\rm Var}\left(\frac{1}{K_r}\sum_{k=1}^{K_r}B_k^i\right)}=\frac{D_{\ell,n}}{2}\!\!\!\sum_{i\in\cM_{\ell,n}}\!\!\sqrt{\frac{1}{K_r^2}{\rm Var}\left(\sum_{k=1}^{K_r}B_k^i\right)}
    \intertext{since $B_k^i$'s are i.i.d. (see Proposition~\ref{prop:concentration})}
    &=\frac{D_{\ell,n}}{2}\sum_{i\in\cM_{\ell,n}}\sqrt{\frac{1}{K_r^2}K_rp_i(1-p_i)}= \frac{D_{\ell,n}}{2\sqrt{K_r}}\sum_{i\in\cM_{\ell,n}}\sqrt{p_i(1-p_i)}
    \intertext{Apply Cauchy--Schwarz to $1$ and $\sqrt{p_i(1-p_i)}$ to obtain:}
    &\leq\frac{D_{\ell,n}}{2\sqrt{K_r}}\sqrt{|\cZ^{\ell,n}|}\sqrt{\sum_{i\in\cM_{\ell,n}}p_i(1-p_i)}\leq\frac{D_{\ell,n}}{2\sqrt{K_r}}\sqrt{|\cZ^{\ell,n}|}\sqrt{\sum_{i\in\cM_{\ell,n}}p_i}\\
    &=\frac{D_{\ell,n}\sqrt{|\cZ_{\ell,n}|}}{2\sqrt{K_r}}.\qquad (\triangle)
\end{align*}}
{\noindent\bf Step-3:} Using the bound given by Proposition~\ref{prop:concentration}, set $\delta = \exp\left(-\frac{2K_r\epsilon^2}{\operatorname{diam}(\cZ^{\ell,n})^2}\right)$. Solving for $\epsilon$ yields $\epsilon = D_{\ell,n}\sqrt{\frac{\log(1/\delta)}{2K_r}}$ where $\log$ is the natural logarithm. Define $\rho_{r}^{\ell,n}(\delta) := \frac{D_{\ell,n}\sqrt{|\cZ_{\ell,n}|}}{2\sqrt{K_r}} + D_{\ell,n}\sqrt{\frac{\log(1/\delta)}{2K_r}}$ as in the statement and note that
\begin{equation}\label{eq:event-comparison}
    \cW_1(\hat\PP_r^{\ell,n},\hat\PP^{\ell,n}) \leq \Big(\cW_1(\hat\PP_r^{\ell,n},\hat\PP^{\ell,n}) - \EE\cW_1(\hat\PP_r^{\ell,n},\hat\PP^{\ell,n})\Big) + \EE\cW_1(\hat\PP_r^{\ell,n},\hat\PP^{\ell,n}).
\end{equation}
Equation \eqref{eq:event-comparison} asserts that the event $\cW_1(\hat\PP_r^{\ell,n},\hat\PP^{\ell,n})$ inherits any bound for the sum of $\cW_1(\hat\PP_r^{\ell,n},\hat\PP^{\ell,n}) - \EE\cW_1(\hat\PP_r^{\ell,n},\hat\PP^{\ell,n})$ and $ \EE\cW_1(\hat\PP_r^{\ell,n},\hat\PP^{\ell,n})$. By using Proposition~\ref{prop:concentration} for the first term and the earlier estimate ($\triangle$) for the second term; the desired claim follows.
\end{proof}

\begin{theorem}\label{thm:quantized}
    Adopting the notation of Lemma~\ref{lem:event}, let
    \[
        \underline{U}_{r,\delta}^{*,\ell,n,(m)}\in\argmin_{\underline{U}\in\UU_m^T}V^{\ell,n}(\hat\PP_r^{\ell,n},\underline{U}).
    \]
    With probability at least $1-\delta$, under Assumption~\ref{assumption:compactness}, we have
    \[
        V^{\ell,n}(\hat\PP^{\ell,n}, \underline{U}_{r,\delta}^{*,\ell,n,(m)}) \leq \inf_{\underline{U}^{(m)}\in\UU_m^T} V^{\ell,n}(\hat\PP^{\ell,n}, \underline{U}^{(m)}) + 2L\rho_r^{\ell,n}(\delta).
    \]
\end{theorem}
\begin{proof}
On the condition that event $\cW_1(\hat\PP_r^{\ell,n},\hat\PP^{\ell,n}) \leq \rho_{r}^{\ell,n}(\delta)$ occurs, let
\[
    \underline{U}^{*,\ell,n,(m)} \in \argmin_{\underline{U}^{(m)}\in\UU_m^T}V^{\ell,n}(\hat\PP^{\ell,n},\underline{U}^{(m)}).
\]
The set $\argmin_{\underline{U}^{(m)}\in\UU_m^T}V^{\ell,n}(\hat\PP^{\ell,n},\underline{U}^{(m)})$ is non-empty since $V^{\ell,n}$ is defined on a finite set. By Lemma~\ref{lem:lipschitz}, we have
\begin{equation}\label{eq:np-1}
    V^{\ell,n}(\hat\PP^{\ell,n}, \underline{U}_{r,\delta}^{*,\ell,n,(m)}) \leq V^{\ell,n}(\hat\PP_r^{\ell,n}, \underline{U}_{r,\delta}^{*,\ell,n,(m)}) + L\rho_r^{\ell,n}(\delta).
\end{equation}
On the other hand, we also have
\begin{align}\label{eq:np-2}
    V^{\ell,n}(\hat\PP_r^{\ell,n}, \underline{U}_{r,\delta}^{*,\ell,n,(m)}) &\leq V^{\ell,n}(\hat\PP_r^{\ell,n}, \underline{U}^{*,\ell,n,(m)})\nonumber\\
    \intertext{also by Lemma~\ref{lem:lipschitz}}
    &\leq V^{\ell,n}(\hat\PP^{\ell,n}, \underline{U}^{*,\ell,n,(m)}) + L\rho_r^{\ell,n}(\delta)\nonumber\\
    &=\inf_{\underline{U}^{(m)}\in\UU_m^T}V^{\ell,n}(\hat\PP^{\ell,n}, \underline{U}^{(m)}) + L\rho_r^{\ell,n}(\delta).
\end{align}

Combining~\eqref{eq:np-1} and \eqref{eq:np-2} results in
\begin{equation}\label{eq:thm9-2}
    V^{\ell,n}(\hat\PP^{\ell,n}, \underline{U}_{r,\delta}^{*,\ell,n,(m)})\leq\inf_{\underline{U}^{(m)}\in\UU_m^T}V^{\ell,n}(\hat\PP^{\ell,n},\underline{U}^{(m)}) + 2L\rho_r^{\ell,n}(\delta).
\end{equation}
Since such inequality holds in case of the occurrence of the event $\cW_1^{\ell,n}(\hat\PP_r^{\ell,n},\hat\PP^{\ell,n}) \leq \rho_{r}^{\ell,n}(\delta)$, we are done, thanks to Lemma~\ref{lem:event}.
\end{proof}

\section{Quantitative Results for the Base Model}\label{sec:base}
In this section, we transfer the results of previous section by giving explicit quantitative bounds for the base model. When switching back to the base model from the quantized one, we do not only inherit the generalization bounds for the quantized model, but also the quantization errors. Therefore, the bound for the base model is coarser, however, this is an expected trade-off since the state is now infinite-dimensional.

\begin{lemma}\label{lem:lipU}
Under Assumption~\ref{assumption:compactness}, there exists a constant $B_T>0$ such that, for all $U,U^\prime\in\UU^T$ and for every $P\in\cP(\cZ)$,
\[
  |V(P,U)-V(P,U^\prime)| \leq B_T \|U-U^\prime\|_{\UU^T},
\]
where $\|U-U^\prime\|_{\UU^T}^2:=\sum_{t=0}^{T-1}\|U_t-U_t'\|_{\UU}^2$. Further, the same estimate, with the same constant, holds for $V^n$ on
$\cP(\cP(\XX_n)\times\cP(\XX_n))$ and for $V^{\ell,n}$ on $\cP(\cP^{(\ell)}(\XX_n)\times\cP^{(\ell)}(\XX_n))$ (i.e. the bound is uniform
in the quantization parameters $\ell,n$).
\end{lemma}
\begin{proof}
By \cite[Proposition~3]{akman2026optimal}, the single-particle map
$f:\XX\times\UU\times\cP(\XX)\to\XX$ is jointly continuous on its domain, which is compact under Assumption~\ref{assumption:compactness}. Moreover, the attention kernel
\[
  (Q,K,V)\mapsto \frac{\int_\mathrm{S}\mu_\mathrm{S}(\de\tilde z)\,\e^{\beta\langle Qx,K\tilde z\rangle}V\tilde z}{\int_\mathrm{S}\mu_\mathrm{S}(\de z)\,\e^{\beta\langle Qx,Kz\rangle}}
\]
is smooth in the weights, with denominator bounded below by $\exp(-\beta C_\mathrm{S})$ for some constant $C_\mathrm{S}<\infty$ depending only on $\mathrm{diam}(S)$ and the bounds on $\|Q\|,\|K\|$. Hence there exists $L_f>0$ such that, for all $(p,x)\in\XX$, all $\mu\in\cP(\XX)$, and all $U,U^\prime\in\UU$,
\begin{equation}\label{eq:f-lip-U}
  \|f((p,x),U,\mu)-f((p,x),U^\prime,\mu)\|_\mathrm{S}  \leq  L_f \|U-U^\prime\|_{\UU}.
\end{equation}
Combining \eqref{eq:f-lip-U} with the measure-Lipschitz estimate of Lemma~\ref{cor:phi-lipschitz}, one obtains, by induction on $t$, a constant $L^\prime_T>0$ such that for all $\mu\in\cP(\XX)$ and all $U,U^\prime\in\UU^T$,
\begin{equation}\label{eq:PhiT-lip-U}
  W_{2,\lambda} \left(\Phi_T(\mu,U),\Phi_T(\mu,U^\prime)\right)
   \leq L^\prime_T \|U-U^\prime\|_{\UU^T}.
\end{equation}
Indeed, denoting $\mu_t:=\Phi_t(\mu,U|_{t-1})$ and $\mu_t':=\Phi_t(\mu,U^\prime|_{t-1})$,
\[
W_{2,\lambda}(\mu_{t+1},\mu_{t+1}') \leq  W_{2,\lambda} \left(\Phi(\mu_t,U_t),\Phi(\mu_t,U_t')\right) + W_{2,\lambda} \left(\Phi(\mu_t,U_t'),\Phi(\mu_t',U_t')\right),
\]
where the first term is bounded by $L_f\|U_t-U_t'\|_\UU$ via \eqref{eq:f-lip-U}, and the second by $L_\Phi W_{2,\lambda}(\mu_t,\mu_t')$ by the measure-Lipschitz estimate. Iterating yields~\eqref{eq:PhiT-lip-U} with $L^\prime_T=L_f\sum_{t=0}^{T-1}L_\Phi^{ T-1-t}$.

Now, fix $P\in\cP(\cZ)$ and $U,U^\prime\in\UU^T$. Using the elementary identity $|a^2-b^2|=|a-b||a+b|$ together with the compactness of $\cP(\XX)$ (which yields a uniform bound $B_W$ on $W_{2,\lambda}(\Phi_T(\mu,U),\nu)+W_{2,\lambda}(\Phi_T(\mu,U^\prime),\nu)$),
\begin{align*}
  |V(P,U)-V(P,U^\prime)| &\leq \int_{\cZ} P(\de\mu,\de\nu)\,\big|W_{2,\lambda}(\Phi_T(\mu,U),\nu)^2-W_{2,\lambda}(\Phi_T(\mu,U^\prime),\nu)^2\big|\\
  &\leq B_W\int_{\cZ} P(\de\mu,\de\nu)\,\big|W_{2,\lambda}(\Phi_T(\mu,U),\nu)-W_{2,\lambda}(\Phi_T(\mu,U^\prime),\nu)\big|\\
  &\leq B_W\int_{\cZ} P(\de\mu,\de\nu)\,W_{2,\lambda}\left(\Phi_T(\mu,U),\Phi_T(\mu,U^\prime)\right)\\
  &\leq B_W L^\prime_T \|U-U^\prime\|_{\UU^T},
\end{align*}
where the third line is the reverse triangle inequality for $W_{2,\lambda}$ and the last is~\eqref{eq:PhiT-lip-U}. Setting $B_T:=B_W L^\prime_T$ proves the bound for $V$. Lastly, note that the arguments leading to~\eqref{eq:f-lip-U} and~\eqref{eq:PhiT-lip-U} are purely at the level of the dynamics $f$ and $\Phi$ and do not interact with the state or measure-state quantization. Consequently the same Lipschitz constant $B_T$ controls $V^n$ and $V^{\ell,n}$ on the corresponding quantized measure spaces, and is independent of $\ell,n$.
\end{proof}

Similar to \cite[Assumption~11]{akman2026optimal}, we need to impose some assumptions on $\ell,n,K_r$ since unstable bounds appearing in Theorem~\ref{thm:quantized} and \cite[Proposition~10]{akman2026optimal} have to be controlled.

\begin{assumption}\label{assumption:quantization}
    Choose $\ell = \ell(n)$ with $\ell \gg \sqrt{|\XX_n|}\sim\sqrt{N}n^{d/2}$ so that
    \[
    \rho_{\ell, n} := \frac{1}{\ell} \sqrt{\frac{\lfloor|\XX_n|/2\rfloor(|\XX_n| - \lfloor|\XX_n|/2\rfloor)}{|\XX_n|}} \leq \frac{1}{\ell}\sqrt{\frac{|\XX_n|}{2}} \xrightarrow[n\to\infty]{} 0^+
    \]
    as $n\to+\infty$. For Theorem~\ref{thm:main}, we further assume that $(K_r)_{r\in\NN}$ is chosen in a way that $K_r \geq C \cdot |\cZ^{\ell,n}|\cdot\log1/\delta$, where $|\cZ^{\ell,n}| := \binom{\ell + |\XX_n| - 1}{|\XX_n|- 1}^2$.
\end{assumption}

The practical implications of these assumptions can be interpreted as: To have a stable generalization regime under approximated training scheme, one needs to (1) increase the quantization rate for measure-states due to higher dimensionality, (2) increase the number of samples relative to the quantization rates to get less penalty from the quantization related errors. So, when both $\ell,n,r$ are sent to infinity, the priority must be $r$ then $\ell$, and finally $n$. The quantization error due to action quantization parameter $m$ dealt with independently in light of the findings of this paper. The majorization $K_r \geq C \cdot |\cZ^{\ell,n}|\cdot\log1/\delta$ governs non-vacuousness of Theorem~\ref{thm:main}.

\begin{lemma}\label{lem:b}
   Under Assumptions~\ref{assumption:compactness} and~\ref{assumption:quantization}, there are positive sequences
   \begin{itemize}
    \item[(i)] $(\alpha_{\ell,n})_{\ell,n\in\NN}$ converging to $0$ as $\ell,n\to+\infty$ such that, for all $\underline{U}^{(m)}\in\UU_m^T$
    \[
        |V(\PP,\underline{U}^{(m)}) - V^{\ell,n}(\hat\PP^{\ell,n},\underline{U}^{(m)})| \leq \alpha_{\ell,n}.
    \]
    \item[(ii)] $(\beta_{\ell,n,m})_{\ell,n,m\in\NN}$ converging to $0$ as $\ell,n,m\to+\infty$ such that
    \[
        \inf_{\underline{U}^{(m)}\in\UU_m^T}V^{\ell,n}(\hat\PP^{\ell,n},\underline{U}^{(m)}) \leq \inf_{\underline{U}\in\UU^T}V(\PP, \underline{U}) + \beta_{\ell,n,m}.
    \]
   \end{itemize}
\end{lemma}

\begin{proof}
\textbf{Part (i).} Fix $\underline{U}^{(m)}\in\UU_m^T$. By the construction of
$\hat\PP^{\ell,n}$ in \eqref{eq:quantized-true}, the law $\hat\PP^{\ell,n}$ on
$\cP^{(\ell)}(\XX_n)\times\cP^{(\ell)}(\XX_n)$ is the push-forward of $\PP$ under the map
\[
  (\mu,\nu)\mapsto\big(R_\ell(\mu\circ Q_n^{-1}),R_\ell(\nu\circ Q_n^{-1})\big),
\]
so by a change of variables
\[
  V^{\ell,n}(\hat\PP^{\ell,n},\underline{U}^{(m)}) = \int_{\cZ}\PP(\de\mu,\de\nu)\,W_{2,\lambda}\Big(\Phi_T^{\ell,n}\big(R_\ell(\mu\circ Q_n^{-1}),\underline{U}^{(m)}\big),\,R_\ell(\nu\circ Q_n^{-1})\Big)^{2}.
\]
Hence
\begin{equation}\label{eq:partI-decomp}
  |V(\PP,\underline{U}^{(m)})-V^{\ell,n}(\hat\PP^{\ell,n},\underline{U}^{(m)})|\leq \int_{\cZ}\PP(\de\mu,\de\nu)\,|A_1^2-A_2^2|,
\end{equation}
where $A_1:=W_{2,\lambda}(\Phi_T(\mu,\underline{U}^{(m)}),\nu)$ and $A_2:=W_{2,\lambda}(\Phi_T^{\ell,n}(R_\ell(\mu\circ Q_n^{-1}),\underline{U}^{(m)}),R_\ell(\nu\circ Q_n^{-1}))$.

By compactness of $\cP(\XX)$ and joint continuity of $W_{2,\lambda}$ and
$\Phi_T$, the sum $A_1+A_2$ is bounded uniformly in $\mu,\nu,\underline{U}^{(m)}$ by a constant which we denote $C_T$. Therefore $|A_1^2-A_2^2|\leq C_T|A_1-A_2|$, and we estimate $|A_1-A_2|$ via two intermediate quantities that are used only for theoretical justification:
\begin{equation*}
  \begin{split}
      |A_1-A_2| &\leq\underbrace{\big|W_{2,\lambda}(\Phi_T(\mu,\underline{U}^{(m)}),\nu)-W_{2,\lambda}(\Phi_T^{\ell,n}(\mu,\underline{U}^{(m)}),\nu)\big|}_{(\star_1)}\\
      &\quad+\underbrace{\big|W_{2,\lambda}(\Phi_T^{\ell,n}(\mu,\underline{U}^{(m)}),\nu) -W_{2,\lambda}(\Phi_T^{\ell,n}(\mu,\underline{U}^{(m)}),R_\ell(\nu\circ Q_n^{-1}))\big|}_{(\star_2)}\\
      &\quad+\underbrace{\begin{split}
          \big|W_{2,\lambda}(\Phi_T^{\ell,n}(\mu,\underline{U}^{(m)}),R_\ell(\nu\circ Q_n^{-1}))\\
          &\hspace{-5em}-W_{2,\lambda}(\Phi_T^{\ell,n}(R_\ell(\mu\circ Q_n^{-1}),\underline{U}^{(m)}),R_\ell(\nu\circ Q_n^{-1}))\big|
      \end{split}}_{(\star_3)}
  \end{split}
\end{equation*}
where (with a slight abuse of notation) $\Phi_T^{\ell,n}(\mu,\cdot)$ denotes
the quantized flow applied to the non-quantized $\mu$ via the canonical embedding $\iota:\XX_n\hookrightarrow\XX$.

For $(\star_1)$, by the reverse triangle inequality and by \cite[Proposition~10 and Lemma~20]{akman2026optimal}
\[
  (\star_1) \leq W_{2,\lambda} \left(\Phi_T(\mu,\underline{U}^{(m)}),\Phi_T^{\ell,n}(\mu,\underline{U}^{(m)})\right)
    \leq \frac{1}{n}+\rho_{\ell,n}.
\]
For $(\star_2)$,
\[
  (\star_2)\leq W_{2,\lambda}(\nu,R_\ell(\nu\circ Q_n^{-1}))\leq W_{2,\lambda}(\nu,\nu\circ Q_n^{-1})+W_{2,\lambda}(\nu\circ Q_n^{-1},R_\ell(\nu\circ Q_n^{-1}))\leq \frac{1}{n}+\rho_{\ell,n}.
\]

For $(\star_3)$, apply reverse triangle inequality and decompose once more:
\begin{equation*}
  \begin{split}
        (\star_3)&\leq W_{2,\lambda}\big(\Phi_T^{\ell,n}(\mu,\underline{U}^{(m)}), \Phi_T^{\ell,n}(R_\ell(\mu\circ Q_n^{-1}), \underline{U}^{(m)})\big)\\
        &\leq W_{2,\lambda}\big(\Phi_T^{\ell,n}(\mu,\underline{U}^{(m)}), \Phi_T(\mu, \underline{U}^{(m)})\big)\\
        &\quad+W_{2,\lambda}\big(\Phi_T(\mu, \underline{U}^{(m)}), \Phi_T(R_\ell(\mu\circ Q_n^{-1}), \underline{U}^{(m)})\big)\\
        &\quad+W_{2,\lambda}\big(\Phi_T(R_\ell(\mu\circ Q_n^{-1}), \underline{U}^{(m)}), \Phi_T^{\ell,n}(R_\ell(\mu\circ Q_n^{-1}), \underline{U}^{(m)})\big)
  \end{split}
\end{equation*}
The first and third terms are bounded as it is done with $(\star_1)$. For the second term, we apply Lemma~\ref{cor:phi-lipschitz} and then $(\star_2)$ argument to get a bound $L_\Phi^T(\frac{1}{n} + \rho_{\ell,n})$. Overall, we obtain
\[
    (\star_3) \leq (2+L_\Phi^T)\bigg(\frac{1}{n}+\rho_{\ell,n}\bigg)
\]

Combining all three, $|A_1-A_2|\leq (4+L_\Phi^T)\big(\frac{1}{n} + \rho_{\ell,n}\big)$, so plugging back into~\eqref{eq:partI-decomp},
\[
  |V(\PP,\underline{U}^{(m)})-V^{\ell,n}(\hat\PP^{\ell,n},\underline{U}^{(m)})|\leq C_T(4 + L_\Phi^T) \left(\frac{1}{n}+\rho_{\ell,n}\right) =: \alpha_{\ell,n}.
\]
This bound is uniform in $\underline{U}^{(m)}\in\UU_m^T$, and $\alpha_{\ell,n}\to 0$ as $\ell,n\to\infty$ since $\rho_{\ell,n}\downarrow 0$.

\medskip

{\bf Part (ii).} We introduce the singly-state-quantized intermediate steps: Let $\hat\PP^n$ be the push-forward of $\PP$ under $(\mu,\nu)\mapsto(\mu\circ Q_n^{-1},\nu\circ Q_n^{-1})$, and let $V^n$ be the value function of the model in which the state space is replaced by $\XX_n$ but the measure-state and action spaces remain non-quantized. We decompose
\begin{align}
  &\inf_{\underline{U}^{(m)}\in\UU_m^T}V^{\ell,n}(\hat\PP^{\ell,n},\underline{U}^{(m)})-\inf_{U\in\UU^T}V(\PP,U)\notag\\
  &\qquad=\bigg(\inf_{\underline{U}^{(m)}\in\UU_m^T}V^{\ell,n}(\hat\PP^{\ell,n},\underline{U}^{(m)})-\inf_{\underline{U}^{(m)}\in\UU_m^T}V^n(\hat\PP^n,\underline{U}^{(m)})\bigg)\label{eq:A}\\
  &\qquad\quad+\bigg(\inf_{\underline{U}^{(m)}\in\UU_m^T}V^n(\hat\PP^n,\underline{U}^{(m)}) -\inf_{U\in\UU^T}V^n(\hat\PP^n,U)\bigg)\label{eq:B}\\
  &\qquad\quad+\bigg(\inf_{U\in\UU^T}V^n(\hat\PP^n,U)-\inf_{U\in\UU^T}V(\PP,U)\bigg).\label{eq:C}
\end{align}
We bound each term using the elementary fact that, for any two real-valued functions $g,h$ on a common domain $\Omega$ admitting infima, $|\inf_{\Omega}g-\inf_{\Omega}h|\leq\sup_{\Omega}|g-h|$.

{\em Term~\eqref{eq:A} (measure-state quantization).} The argument of Part~(i) gives
\[
  \sup_{\underline{U}^{(m)}\in\UU_m^T}|V^{\ell,n}(\hat\PP^{\ell,n},\underline{U}^{(m)})-V^n(\hat\PP^n,\underline{U}^{(m)})|
   \leq  C_T(4+L_\Phi^T)\rho_{\ell,n}.
\]
Indeed, the only difference between $\hat\PP^{\ell,n}$ and $\hat\PP^n$ is the application of $R_\ell$, contributing $\rho_{\ell,n}$ in each of $(\star_1)$ and $(\star_2)$.

{\em Term~\eqref{eq:C} (state-space quantization).} Likewise, running the Part~(i) argument with only the state-quantization step yields
\[
  \sup_{\underline{U}\in\UU^T}|V^n(\hat\PP^n,\underline{U})-V(\PP,\underline{U})| \leq C_T(4+L_\Phi^T)\frac{1}{n}.
\]

{\em Term~\eqref{eq:B} (action-space quantization).} By Lemma~\ref{lem:value-function-is-continuous} (that is, joint continuity of $V$, applied to $V^n$ on the compact $\cP(\cP(\XX_n)\times\cP(\XX_n))\times\UU^T$), the function $\underline{U}\mapsto V^n(\hat\PP^n,\underline{U})$ is continuous on the compact set $\UU^T$, so it admits a minimizer $\underline{U}^*\in\argmin_{\underline{U}\in\UU^T}V^n(\hat\PP^n,\underline{U})$. By the action-quantization, for each $t\in\mathcal T$ there exists $\tilde U_t^{(m)}\in\UU_m$ with $\|U_t^*-\tilde U_t^{(m)}\|_{\UU}\leq 1/m$, hence
\[
    \|\underline{U}^*-\tilde{\underline{U}}^{(m)}\|_{\UU^T} = \bigg(\sum_{t=0}^{T-1}\|U_t^*-\tilde{U_t}^{(m)}\|_\UU^2\bigg)^{1/2} \leq \frac{\sqrt T}{m}.
\]
Applying Lemma~\ref{lem:lipU} to $V^n$,
\begin{equation*}
    \begin{split}
         \inf_{\underline{U}^{(m)}\in\UU_m^T}V^n(\hat \PP^n,\underline{U}^{(m)})\leq V^n(\hat \PP^n,\tilde{\underline{U}}^{(m)}) &\leq V^n(\hat \PP^n,\underline{U}^*)+B_T \frac{\sqrt T}{m}\\
         &= \inf_{\underline{U}\in\UU^T}V^n(\hat \PP^n,\underline{U})+B_T \frac{\sqrt T}{m}.
    \end{split}
\end{equation*}
The reverse inequality $\inf_{\underline{U}\in\UU^T}V^n(\hat\PP^n,\underline{U})\leq\inf_{\underline{U}^{(m)}\in\UU_m^T}V^n(\hat\PP^n,\underline{U}^{(m)})$ is immediate from $\UU_m^T\subseteq\UU^T$. Therefore
\[
  \bigg|\inf_{\underline{U}^{(m)}\in\UU_m^T}V^n(\hat\PP^n,\underline{U}^{(m)})-\inf_{\underline{U}\in\UU^T}V^n(\hat\PP^n,\underline{U})\bigg|
   \leq B_T \frac{\sqrt T}{m}.
\]

Summing the bounds on \eqref{eq:A}---\eqref{eq:C} via the triangle inequality,
\begin{equation*}
    \bigg|\inf_{\underline{U}^{(m)}\in\UU_m^T}V^{\ell,n}(\hat\PP^{\ell,n},\underline{U}^{(m)})-\inf_{\underline{U}\in\UU^T}V(\PP,\underline{U})\bigg| \leq \underbrace{C_T(4+L_\Phi^T)\left(\rho_{\ell,n}+\frac{1}{n}\right)+B_T \frac{\sqrt T}{m}}_{=:\beta_{\ell,n,m}}
\end{equation*}
Since $\rho_{\ell,n}\downarrow 0$, the right-hand side vanishes as $\ell,n,m\to\infty$, completing the proof.
\end{proof}

\begin{theorem}\label{thm:main}
    Under Assumptions~\ref{assumption:compactness} and \ref{assumption:quantization}, and adopting the notation of Theorem~\ref{thm:quantized}, for any $\delta\in(0,1)$ and all $\ell,n,m\in\NN$, with probability at least $1-\delta$, we have
    \begin{equation*}
        \begin{split}
            V(\PP,\underline{U}_{r,\delta}^{*,\ell,n,(m)}) \leq \inf_{\underline{U}\in\UU^T}V(\PP, \underline{U}) &+ 2L\bigg(\frac{D_{\ell,n}\sqrt{|\cZ_{\ell,n}|}}{2\sqrt{K_r}} + D_{\ell,n}\sqrt{\frac{\log(1/\delta)}{2K_r}}\bigg)\\
            &+ 2C_T(4 + L_\Phi^T) \left(\frac{1}{n}+\rho_{\ell,n}\right) +B_T \frac{\sqrt T}{m}.
        \end{split}
    \end{equation*}
\end{theorem}
\begin{proof}
    We decompose
    \begin{align*}
        V(\PP,\underline{U}_{r,\delta}^{*,\ell,n,(m)}) - \inf_{\underline{U}\in\UU^T}V(\PP, \underline{U}) &= V(\PP,\underline{U}_{r,\delta}^{*,\ell,n,(m)}) - V^{\ell,n}(\hat\PP^{\ell,n}, \underline{U}_{r,\delta}^{*,\ell,n,(m)})\\
        &\quad+V^{\ell,n}(\hat\PP^{\ell,n}, \underline{U}_{r,\delta}^{*,\ell,n,(m)})-\inf_{\underline{U}^{(m)}\in\UU_m^T}V^{\ell,n}(\hat\PP^{\ell,n}, \underline{U}^{(m)})\\
        &\quad+\inf_{\underline{U}^{(m)}\in\UU_m^T}V^{\ell,n}(\hat\PP^{\ell,n}, \underline{U}^{(m)})-\inf_{\underline{U}\in\UU^T}V(\PP, \underline{U}).
    \end{align*}
    Using Lemma~\ref{lem:b}, the first term is bounded $\alpha_{\ell,n}$, and the third is bounded by $\beta_{\ell,n,m}$. The second term is dominated by $2L\rho_{r}^{\ell,n}(\delta)$ due to Theorem~\ref{thm:quantized}, while the event that this ordering occurs happens with probability at least $1-\delta$.
\end{proof}
\ka{\begin{remark}
    For a fixed quantization configuration $(n,\ell,m)\in\NN^3$, the generalization bound obtained in Theorem~\ref{thm:quantized} achieves the optimal rate $\cO(K_r^{-1/2})$ with respect to the sample count in the data set. If we choose $K_r = (LD_{\ell,n})^2|\cZ^{\ell,n}|\log(1/\delta)$, the sample complexity bound (without approximation error due to quantization) becomes $\frac{1}{\log(1/\delta)} + \sqrt{\frac{2}{|\cZ^{\ell,n}|}}$. For a fixed $\delta\in(0,1)$, under Assumption~\ref{assumption:quantization}, the generalization error is of order $\frac{1}{\sqrt{|\cZ^{\ell,n}|}}$.
\end{remark}}
\ka{\begin{remark}
    Usual generalization problems are studied at the level of measures on particles, in which it is assumed that the data originate from a probability measure $\xi$ on $(\XX)^N$. For a sequence of particles $(x^i)_{i=1}^N\in(\XX)^N$, if we apply lifting $\cL:(\XX)^N\to\cP_{\rm E}(\XX);(x^1,\ldots,x^N)\mapsto\frac{1}{N}\sum_{i=1}^N\delta_{x^i}$, then $\xi\circ \cL^{-1}$ lives inside $\cP_{\rm E}(\XX)$--the space of empirical measures. However, our results so far and henceforth are based on the assumption that the unknown underlying distribution is from $\cP(\cP(\XX)\times\cP(\XX))$, which is more general.
\end{remark}}
\begin{remark}[Dependence to depth and width]
    The final bound obtained by Theorem~\ref{thm:main} clearly indicates dependence on the number of layers $T$ on different parts. First of all, we have an explicit term $\sqrt{T}$. Moreover, our Lipschitz estimates depend on $T$ as well. In particular, the proofs of Lemmata~\ref{lem:lipschitz} and~\ref{lem:lipU} show that the Lipschitz bounds for the value functional $V$ scale with $T$ exponentially.
    
    Although, the dependence on width is not explicitly stated, it appears implicitly through (a) the state covering $\XX_n$, which scales according to the state space dimension $d$; (b) diameter $\operatorname{diam}(\mathrm{S})$; (c) various constants $B_W, C_\mathrm{S}, C_T$ depending on $\operatorname{diam}(\mathrm{S})$.
\end{remark}

\begin{remark}[Expanded forms of the constants in Theorem~\ref{thm:main}]
    The generalization bound found in Theorem~\ref{thm:main} involves three blackbox constants, namely $C_T$, $B_T$ and $L$. We now give explicit formulae for those that involve terms intrinsic to the Transformer architecture.

    The term $C_T$ appears in the proof of Lemma~\ref{lem:b}, and satisfies
    \begin{equation}\label{eq:c_t}
        \begin{split}
            C_T &\geq A_1+A_2\\
            &:=W_{2,\lambda}(\Phi_T(\mu,\underline{U}^{(m)}),\nu)+W_{2,\lambda}(\Phi_T^{\ell,n}(R_\ell(\mu\circ Q_n^{-1}),\underline{U}^{(m)}),R_\ell(\nu\circ Q_n^{-1})),
        \end{split}
    \end{equation}
    where $\mu,\nu\in\cP(\XX)$ and $\underline{U}^{(m)}\in\UU^T$.

    $B_T$ is defined as the Lipschitz constant of $V$ with respect to weight trajectory in Lemma~\ref{lem:lipU}. A more explicit, intermediate form is $B_T = B_WL_T'$. Here, $L_T'$ is defined as $L_f\sum_{t=0}^{T-1}L_\Phi^{T-1-t}$ in Lemma~\ref{lem:lipU}. In the proof of Lemma~\ref{cor:phi-lipschitz} (see \cite[Corollary~19]{akman2026optimal}), it is found that $L_\Phi = \sqrt{2}L_f$. In \cite[Lemma~18]{akman2026optimal}, $L_f$ is calculated as follows:
    \begin{equation}\label{eq:l_f}
        \max\left\{ 1, \frac{\mathrm{diam(\UU)}}{C_{\rm min}}\sqrt{2(L_DB_N + L_NB_D)} \right\}
    \end{equation}
    where
    \begin{align*}
        &L_D := \displaystyle\max\left\{\sup_{Q,K,Z,z'}\mathrm{Lip}\left( x\mapsto\e^{\beta\langle Qx, Kz'\rangle}\right)\mathrm{diam}(\mathrm{S})\mathrm{diam}(\UU),\sup_{Q,K,z'}|\e^{\beta\langle Qx, Kz'\rangle}|\right\}\\
        &B_N := \e^{\beta(\mathrm{diam}(U)\,\mathrm{diam}(\mathrm{S}))^2}\mathrm{diam}(U)\,\mathrm{diam}(\mathrm{S})\\
        &L_N := \displaystyle\max\bigg\{\sup_{Q,K,Z,z'}\!\!\mathrm{Lip}\left( x\mapsto\e^{\beta\langle Qx, Kz'\rangle}Vz'\right)\mathrm{diam}(\mathrm{S})\mathrm{diam}(\UU),\\
        &\hspace{24em}\sup_{Q,K,V,z'}\|\e^{\beta\langle Qx, Kz'\rangle}Vz'\|\bigg\}\\
        &B_D := \e^{\beta(\mathrm{diam}(U)\,\mathrm{diam}(\mathrm{S}))^2}\\
        &C_{\rm min} := \displaystyle\min_{Q,K,x}\big|\int_{\mathrm{S}}\mu_{\mathrm{S}}(\de z')\,\e^{\beta\langle Qx, Kz'\rangle}\big|
    \end{align*}
    Since $Q,K,z$ live in compact spaces, $C_{\rm min}$ has to be strictly positive. Also, two suprema of Lipschitz constants appearing in $L_D$ and $L_N$ are uniformly bounded since they are dependent on weights $Q,K,V$, and the weight space $\UU$ is assumed to be compact (Assumption~\ref{assumption:compactness}).

    The state-Lipschitz constant $L$ is obtained in Lemma~\ref{lem:lipschitz}, and is of the form $B_W\max\{1,L_{\Phi}^T\}$ where $B_W$ can be taken $2\cdot\mathrm{diam}(\cP(\XX), W_{2,\lambda})$. Since
    \begin{align*}
        \|x-y\|_2^2 + \lambda |p-q|^2 &\leq \operatorname{diam}(\mathrm{S}, \|\cdot\|_2)^2 + \lambda|p-q||p+q|\\
        &\leq \operatorname{diam}(\mathrm{S}, \|\cdot\|_2)^2 + 2\lambda,
    \end{align*}
    it follows that $\mathrm{diam}(\cP(\XX), W_{2,\lambda}) = \sqrt{\operatorname{diam}(\mathrm{S}, \|\cdot\|_2)^2 + 2\lambda}$.
\end{remark}

\section{Distributionally Robust Approach to Generalization Problem}\label{sec:dr}

Sections~\ref{sec:quantized} and \ref{sec:base} bounded the excess risk of the empirical training minimizer using Lipschitz estimate of Lemma~\ref{lem:lipschitz} and the concentration inequality for the lifted law in $\cW_1$ of Proposition~\ref{prop:concentration}. In this section, we present a complementary approach: The same machinery provided by Section~\ref{sec:formulation} yields a distributionally robust control formulation of the training problem which can be seen as a zero-sum game of a controller hedging against an adversary choosing a data-generating law from a Wasserstein ball around empirical law. The viewpoint is conceptually distinct from the bound found in Sections~\ref{sec:quantized} and \ref{sec:base} since we now investigate robustness to prior misspecification, instead of finite-sample excess risk bounds.

First, we prove the existence of "robust minimizers" for the described minimax problem. Then, we show that associated robust objectives are $\Gamma$-convergent to training objective and Painlev\'e-Kuratowski upper limit the set of robust minimizers converges to a subset of training minimizers, both of these convergence notions are in probability. As an additional discussion point, we derive finite-sample excess risk bounds for training problem under robust minimizers and observe that the robust minimizers share the same upper generalization bounds with the training minimizers.

Recall that, for each $r\in\NN$, the data set $\cD_r:=\{(\mathbf{x}^{k,r}, \mathbf{y}^{k,r}) : k=1,\ldots,K_r\}$ is a sampling of the true data law $\PP$, which gives rise to empirical measures $\mu^{k,r} := \frac{1}{N}\sum_{i=1}^N\delta_{(p_i, x^{i,k,r})}$ and $\nu^{k,r} := \frac{1}{N}\sum_{i=1}^N\delta_{(p_i, y^{i,k,r})}$ where $\mathbf{x}^{k,r} := (x^{i,k,r})_{i=1}^N$ and $\mathbf{y}^{k,r} := (y^{i,k,r})_{i=1}^N$. Then we lift once more to get the empirical law on input-output pairs
\[
    \PP_r := \frac{1}{K_r}\sum_{k=1}^{K_r}\delta_{(\mu^{k,r},\nu^{k,r})}\in\cP(\cZ)
\]
where $\cZ := \cP(\XX)^{\times 2}$. Given $\rho > 0$ and $r\in\NN$, we define the {\bf ambiguity set} $A_\rho(\PP_r) := \{P\in\cP(\cZ) : \cW_1(P, \PP_r) \leq \rho\}$. The following proposition establishes the existence of {\it robust minimizers}, which are, in general, different from the minimizers obtained from the training procedure.

\begin{proposition}\label{prop:existence_of_robust_min}
    For any $\rho>0$ and $r\in\NN$, there exists $\underline{U}_{\rho,r}^*\in\UU^T$ such that
    \begin{equation}\label{eq:dro}
        \sup_{P\in A_\rho(\PP_r)}V(P,\underline{U}_{\rho, r}^*) = \inf_{\underline{U}\in\UU^T}\sup_{P\in A_\rho(\PP_r)}V(P,\underline{U}).
    \end{equation}
\end{proposition}
\begin{proof}
    Lemma~\ref{lem:value-function-is-continuous} already establishes the continuity of the value function on its domain. By \cite[Lemma~2.41]{aliprantis2006infinite}, $\UU^T\ni\underline{U}\mapsto\sup_{P\in A_\rho(\PP_r)}V(P,\underline{U})$ is lower semicontinuous. Since $\UU^T$ is compact as it is a product of compact spaces (see Assumption~\ref{assumption:compactness}), the claim follows from Weierstrass' extreme value theorem for lower-semicontinuous functions on compact spaces \cite[Theorem~2.43]{aliprantis2006infinite}.
\end{proof}

Equation~\eqref{eq:dro} in Proposition~\ref{prop:existence_of_robust_min} can be seen as a zero-sum game between a player trying to choose --in a sense-- {\em worst} prior (in most applications, initial data) that maximizes the game value, whereas the competing player tries to decrease the game value by choosing an action that performs {\em well enough} for all priors.

\begin{lemma}\label{lem:robust-lem1}
    Suppose that the true distribution $\PP$ belongs to $A_\rho(\PP_r)$ for some $\rho>0$ and $r\in\NN$. Then
    \begin{equation}\label{eq:prop1}
        V(\PP, \underline{U}_{\rho,r}^*) \leq \inf_{\underline{U}\in\UU^T}V(\PP_r, \underline{U}) + L\rho.
    \end{equation}
\end{lemma}
\begin{proof}
    Proposition~\ref{prop:existence_of_robust_min} gives us $V(\PP,\underline{U}_{\rho,r}^*) \leq \inf_{\underline{U}\in\UU^T}\sup_{P\in A_\rho(\PP_r)}V(P,\underline{U})$. By Lemma~\ref{lem:lipschitz}, for all $P\in A_\rho(\PP_r)$ and all $\underline{U}\in\UU^T$, we get
    \begin{equation}
        V(P, \underline{U}) \leq V(\PP_r, \underline{U}) + L \cW_1(P,\PP_r).
    \end{equation}
    In particular,
    \begin{equation}
        \sup_{Q\in A_\rho(\PP)}V(Q, \underline{U}) \leq V(\PP_r, \underline{U}) + L \cW_1(P,\PP_r) \leq V(\PP_r, \underline{U}) + L\rho ,
    \end{equation}
    and consequently,
    \begin{equation}\label{eq:prop2}
        \sup_{Q\in A_\rho(\PP)}V(Q, \underline{U}_{\rho, r}^*) \leq \inf_{\underline{U}\in\UU^T}V(\PP_r, \underline{U}) + L \cW_1(P,\PP_r)\leq \inf_{\underline{U}\in\UU^T}V(\PP_r, \underline{U}) + L \rho,
    \end{equation}
    Combining \eqref{eq:prop2} and by the assumption, we obtain
    \[
        V(\PP, \underline{U}_{\rho, r}^*) \leq \inf_{\underline{U}\in\UU^T}V(\PP_r, \underline{U}) + L \rho.
    \]
\end{proof}

In the next result, we establish the partial asymptotic equivalence of robust minimizers to training minimizers. First, we prove that robust objective is asymptotically equivalent to the training objective. Such convergence is often termed as {\em $\Gamma$-convergence} \cite{de1975tipo} (also see the Chapter~4 in \cite{dal2012introduction}). In the same theorem, we also prove the set of robust minimizer has its upper limit living inside the set of training minimizers. The convergence notion here is named as {\em upper Painlev\'e-Kuratowski convergence} of sets, see Chapter~4-B in \cite{rockafellar1998variational} for a thorough discussion and treatment.

\begin{theorem}\label{thm:robust-convergence}
    Suppose $(\PP_r)_{r\in\NN}$ is an empirical measure process induced by i.i.d. sampled data sets $(\cD_r)_{r\in\NN}$ where $\cD_r$ has $K_r$ samples.  Let $(\rho_r)_{r\in\NN}$ be a monotonically vanishing positive sequence, i.e. $\rho_r\searrow 0^+$, such that $\Prob\{\PP\in A_{\rho_r}(\PP_r)\}\xrightarrow{r\to+\infty} 1$. Then, in probability,
    \begin{itemize}
        \item[(i)] (\;$\Gamma$-convergence of robust objective to training objective)
        \[
            \inf_{\underline{U}\in\UU^T}\sup_{P\in A_{\rho_r}(\PP_r)}V(P, \underline{U}) \to \inf_{\underline{U}\in\UU^T}V(\PP, \underline{U}).
        \]
        \item[(ii)] (Painlev\'e-Kuratowski upper limit of the sets of robust minimizers lies inside the set of training minimizers)
        \[
             \limsup_{r\to+\infty}\,\argmin_{\underline{U}\in\UU^T}\sup_{P\in A_{\rho_r}(\PP_r)}V(P, \underline{U})\subseteq\argmin_{\underline{U}\in\UU^T}V(\PP, \underline{U}).
        \]
    \end{itemize}
\end{theorem}
\begin{proof}
    {\bf Part (i).} By Lemma~\ref{lem:value-function-is-continuous} and the compactness of $\UU^T$, there exists $\underline{U}^*\in\UU^T$ such that $\inf_{\underline{U}\in\UU^T}V(\PP, \underline{U}) = V(\PP,\underline{U}^*)$. Define the event $E_r := \{\PP\in A_{\rho_r}(\PP_r)\}$. By assumption, we know that $\Prob\,E_r \xrightarrow{r\to+\infty} 1$.

    On event $E_r$, as $\PP\in A_{\rho_r}(\PP_r)$, we have $\sup_{Q\in A_{\rho_r}(\PP_r)}V(Q,\underline{U}) \geq V(\PP,\underline{U}) \geq V(\PP,\underline{U}^*)$ for all $\underline{U}\in\UU^T$. Therefore,
    \begin{equation}\label{eq:cor6-1}
        \inf_{\underline{U}\in\UU^T}\sup_{Q\in A_{\rho_{r}}(\PP_{r})}V(Q, \underline{U}) \geq V(\PP,\underline{U}^*).
    \end{equation}
    Fix $Q\in A_{\rho_r}(\PP_{r})$, because of triangle inequality, the following holds:
    \[
        \cW_1(Q,\PP) \leq \cW_1(Q,\PP_{r}) + \cW_1(\PP_{r},\PP) \leq 2\rho_{r}.
    \]
    Invoking Lemma~\ref{lem:lipschitz} gets us to $V(Q,\underline{U}^*) \leq V(\PP,\underline{U}^*) + 2L\rho_{r}$, which implies
    \begin{equation}\label{eq:cor6-2}
         \inf_{\underline{U}^\prime\in\UU^T}\sup_{Q\in A_{\rho_{r}}(\PP_{r})}V(Q, \underline{U}^\prime)\leq V(\PP,\underline{U}^*) + 2L\rho_{r}.
    \end{equation}
    Combining \eqref{eq:cor6-1} and \eqref{eq:cor6-2} implies that, given $\epsilon > 0$, along with the assumption that $\rho_r\to0^+$, we can choose $r_0\in\NN$ such that $2L\rho_{r_0} < \epsilon$ for all $r>r_0$. Then, for any $r>r_0$, we get
    \[
        \Prob\bigg\{\bigg|\inf_{\underline{U}\in\UU^T}\sup_{P\in A_{\rho_r}(\PP_r)}V(P, \underline{U}) - \inf_{\underline{U}\in\UU^T}V(\PP, \underline{U})\bigg| \leq \epsilon\bigg\} \geq \Prob\,E_r\xrightarrow{r\to+\infty} 1.
    \]

    {\bf Part (ii).} For any $r\in\NN$, on the occurrence of $E_r$, apply Lemma~\ref{lem:robust-lem1} with $\rho = \rho_r$:
    \begin{equation}\label{eq:cor2-3}
        V(\PP, \underline{U}_{\rho_r, r}^*) \leq \inf_{\underline{U}\in\UU^T}V(\PP,\underline{U}) + L\rho_r.
    \end{equation}
    Since $E_r$ holds, using Lemma~\ref{lem:lipschitz}, for any $\underline{U}^*\in\argmin_{\underline{U}\in\UU^T}V(\PP,\underline{U})$, one has
    \begin{equation}\label{eq:cor2-4}
        V(\PP_r, \underline{U}^*) \leq V(\PP, \underline{U}^*) + L\rho_r
    \end{equation}
    By combining \eqref{eq:cor2-3} and \eqref{eq:cor2-4}, we get
    \begin{equation}\label{eq:cor2-5}
        V(\PP,\underline{U}_{\rho_r, r}^*) \leq V(\PP,\underline{U}^*) + 2L\rho_r.
    \end{equation}
    for any $\underline{U}_{\rho_r, r}^*\in\argmin_{\underline{U}\in\UU^T}\sup_{P\in A_{\rho_r}(\PP_r)}V(P, \underline{U})$.

    Fix an open set $\cU\subseteq\UU^T$ containing $\argmin_{\UU^T}V(\PP,\cdot)$ and define $\cC := \UU^T\setminus\cU$. Since $\cC$ is a complement of an open set, and thereby, a closed subset of a compact Hausdorff space $\UU^T$, $\cC$ is compact as well. Using the continuity of $V$, (see Lemma~\ref{lem:value-function-is-continuous}), it follows that $V(\PP,\cdot)$ attains its minimum on $\cC$. Moreover, since $\cC\cap\argmin_{\UU^T}V(\PP,\cdot) = \emptyset$ by construction, we get $\min_{\cC}V(\PP,\cdot) > \min_{\cU}V(\PP,\cdot)$. We denote the quantity $\min_{\cC}V(\PP,\cdot) - \min_{\cU}V(\PP,\cdot)$ by $\epsilon$, which is strictly positive.

    As $\rho_r$ monotonically decreases to $0$, it is possible to find $r_0\in\NN$ with the property that $2L\rho_r < \epsilon$ for all $r\geq r_0$. Then, for all $r\geq r_0$, from \eqref{eq:cor2-5}, we get
    \[
        V(\PP,\underline{U}_{\rho_r, r}^*) \leq V(\PP,\underline{U}^*) + 2L\rho_r < \min_{\cU}V(\PP,\cdot) + \epsilon = \min_{\cC}V(\PP,\cdot).
    \]
    From here, it follows that $\underline{U}_{\rho_r, r}^*$ cannot be an element of $\cC$ since it performs strictly better than the elements of $\cC$. Subsequently, $\underline{U}_{\rho_r, r}^* \in \cU$ for large enough $r$. As a consequence, we get
    \[
       \limsup_{r\to+\infty}\,\argmin_{\underline{U}\in\UU^T}\sup_{P\in A_{\rho_r}(\PP_r)}V(P, \underline{U})\subseteq\cU
    \]
    in probability. Since $\cU$ is an arbitrary open set covering $\argmin_{\underline{U}\in\UU^T}V(\PP, \underline{U})$, the claim follows.
\end{proof}

\subsection{Finite-sample excess empirical risk bounds for robust minimizers}

We now derive generalization bounds for the empirical risk under robust minimizers. Our findings indicate that robust minimizers perform equally with training minimizers for worst-case generalization performance.

\begin{theorem}\label{thm:r-quantized}
    Under Assumption~\ref{assumption:compactness}, for all $\delta\in(0,1)$, define
    $$
        \rho_{r}^{\ell,n}(\delta) := \frac{D_{\ell,n}\sqrt{|\cZ_{\ell,n}|}}{2\sqrt{K_r}} + D_{\ell,n}\sqrt{\frac{\log(1/\delta)}{2K_r}}
    $$
    where $D_{\ell,n} := \operatorname{diam}(\cZ_{\ell,n})$ for each $\ell,n\in\NN$. Let
    $$
    \underline{U}_{r,\delta}^{*,\ell,n,(m)}\in\argmin_{\underline{U}^{(m)}\in\UU_m^T}\sup_{Q\in A_{\rho_r^{\ell,n}(\delta)}(\hat\PP_r^{\ell,n})}V^{\ell,n}(Q,\underline{U})
    $$
    where $A_{\rho_r^{\ell,n}(\delta)}(\hat\PP_r^{\ell,n}) := \{Q\in\cP(\cZ_{\ell,n}) : \cW_1(Q, \hat\PP_r^{\ell,n})\leq\rho_r^{\ell,n}(\delta)\}$. Then, with probability at least $1-\delta$, we have
    \[
        V^{\ell,n}(\hat\PP^{\ell,n}, \underline{U}_{r,\delta}^{*,\ell,n,(m)}) \leq \inf_{\underline{U}^{(m)}\in\UU_m^T} V^{\ell,n}(\hat\PP^{\ell,n}, \underline{U}^{(m)}) + 2L\rho_r^{\ell,n}(\delta).
    \]
\end{theorem}
\begin{proof}
Since $V^{\ell,n}(\hat\PP^{\ell,n}, \cdot)$ is a map on a finite set, it has at least one minimizer. Let $\underline{\tilde{U}}^{*,\ell,n,(m)}\in\argmin_{\underline{U}^{(m)}\in\UU_m^T}V^{\ell,n}(\hat\PP^{\ell,n}, \underline{U}^{(m)})$. Now, fix $Q\in A_{\rho_r^{\ell,n}(\delta)}(\hat\PP^{\ell,n}_r)$, then triangle inequality for $\cW_1$ yields
\[
    \cW_1^{\ell,n}(Q,\hat\PP^{\ell,n}) \leq \cW_1^{\ell,n}(Q,\hat\PP_r^{\ell,n})+\cW_1^{\ell,n}(\hat\PP_r^{\ell,n},\hat\PP^{\ell,n}) \leq 2\rho_r^{\ell,n}(\delta).
\]
On the event $\cE = \{\cW_1^{\ell,n}(\hat\PP_r^{\ell,n}, \hat\PP^{\ell,n}) \leq \rho_r^{\ell,n}(\delta)\}$, we have the following inequality chain
\begin{equation}\label{eq:ex-1}
    \begin{split}
        V^{\ell,n}(\hat\PP^{\ell,n},\underline{U}_{r,\delta}^{*,\ell,n,(m)}) &\leq \sup_{Q\in A_{\rho_r^{\ell,n}(\delta)}(\hat\PP_r^{\ell,n})}V^{\ell,n}(Q,\underline{U}_{r,\delta}^{*,\ell,n,(m)})\\
        &\leq \sup_{Q\in A_{\rho_r^{\ell,n}(\delta)}(\hat\PP_r^{\ell,n})}V^{\ell,n}(Q,\underline{\tilde{U}}^{*,\ell,n,(m)})\\
        &\leq V^{\ell,n}(\hat\PP^{\ell,n}, \underline{\tilde{U}}^{*,\ell,n,(m)}) + 2L\rho_r^{\ell,n}(\delta) \qquad (\text{by Lemma~\ref{lem:lipschitz} and $\cE$})\\
        &=\inf_{\underline{U}^{(m)}\in\UU_m^T}V^{\ell,n}(\hat\PP^{\ell,n}, \underline{U}^{(m)}) + 2L\rho_r^{\ell,n}(\delta).
    \end{split}
\end{equation}
The last inequality holds because of both $Q,\hat\PP^{\ell,n}$ belong to $A_{\rho_r^{\ell,n}(\delta)}$. Using Lemma~\ref{lem:lipschitz}, since such inequality holds in case of the occurrence of the event $\cW_1^{\ell,n}(\hat\PP_r^{\ell,n},\hat\PP^{\ell,n}) \leq \rho_{r}^{\ell,n}(\delta)$, which occurs with probability at least $1-\delta$ due to Lemma~\ref{lem:event}, we are done.
\end{proof}

\begin{theorem}\label{thm:r-main}
    Under Assumption~\ref{assumption:compactness}, and adopting the notation of Theorem~\ref{thm:r-quantized}, for any $\delta\in(0,1)$ and all $\ell,n,m\in\NN$ chosen in the way described in Assumption~\ref{assumption:quantization}, with probability at least $1-\delta$, we have
    \[
        V(\PP,\underline{U}_{r,\delta}^{*,\ell,n,(m)}) \leq \inf_{\underline{U}\in\UU^T}V(\PP, \underline{U}) + 2L\rho_{r}^{\ell,n}(\delta) + \alpha_{\ell,n} + \beta_{\ell,n,m}.
    \]
\end{theorem}
\begin{proof}
    We decompose
    \begin{align*}
        V(\PP,\underline{U}_{r,\delta}^{*,\ell,n,(m)}) - \inf_{\underline{U}\in\UU^T}V(\PP, \underline{U}) &= V(\PP,\underline{U}_{r,\delta}^{*,\ell,n,(m)}) - V^{\ell,n}(\hat\PP^{\ell,n}, \underline{U}_{r,\delta}^{*,\ell,n,(m)})\\
        &\quad+V^{\ell,n}(\hat\PP^{\ell,n}, \underline{U}_{r,\delta}^{*,\ell,n,(m)})-\inf_{\underline{U}^{(m)}\in\UU_m^T}V^{\ell,n}(\hat\PP^{\ell,n}, \underline{U}^{(m)})\\
        &\quad+\inf_{\underline{U}^{(m)}\in\UU_m^T}V^{\ell,n}(\hat\PP^{\ell,n}, \underline{U}^{(m)})-\inf_{\underline{U}\in\UU^T}V(\PP, \underline{U}).
    \end{align*}
    Using Lemma~\ref{lem:b}, the first term is bounded $\alpha_{\ell,n}$, and the third is bounded by $\beta_{\ell,n,m}$. The second term is dominated by $2L\rho_{r}^{\ell,n}(\delta)$ due to Theorem~\ref{thm:r-quantized}, while the event that this ordering occurs happens with probability at least $1-\delta$.
\end{proof}

\section{Conclusion}

We have presented a distributionally robust control approach to the generalization problem of Transformers. Building on the doubly lifted measure-valued MDP formulation of \cite{akman2026optimal}, we viewed data sets as probability distributions on the empirical measures of inputs and outputs. Then, we treated the generalization problem as a zero-sum game between a controller trying to choose a weight profile over the layers and an adversary choosing a nearby probability law. This formulation yielded results on three main directions: The existence of a robust minimizer (Proposition~\ref{prop:existence_of_robust_min}), Lipschitz estimates for the expected value function and asymptotic robustness guarantees (Lemma~\ref{lem:lipschitz}; Lemma~\ref{lem:lipU}), a finite-sample bound for the triply quantized model (Theorem~\ref{thm:quantized}), and finally, the transfer of this bound to the base model (Theorem~\ref{thm:main}).

We find some limitations worth emphasizing. The compactness assumption on the action space $\UU$ corresponds to a norm-constraint training scheme --e.g. an operator norm bound constraint-- rather than to unconstrained stochastic gradient descent, thereby, should not be mistaken as a guarantee for an arbitrary Transformer. Moreover, the bound is asymptotic in the prescribed regime $K_r \gg \ell \gg n \to +\infty$, and exhibiting curse of dimensionality through the state covering $\XX_n \sim Nn^d$.

\bibliographystyle{siamplain}

\end{document}